\documentclass{article}

\usepackage[accepted]{icml2025}

\usepackage{microtype}
\usepackage{graphicx}
\usepackage{subfigure}
\usepackage{booktabs} 
\usepackage{hyperref}
\makeatletter
\@ifundefined{theHalgorithm}{}{}
\makeatother

\usepackage{amsmath}
\usepackage{amssymb}
\usepackage{mathtools}
\usepackage{amsthm}

\usepackage[capitalize,noabbrev]{cleveref}

\usepackage{dsfont}
\usepackage[mathscr]{euscript}
\usepackage{xcolor}
\usepackage{colortbl}
\usepackage{thmtools}
\usepackage{thm-restate}

\usepackage{mathtools}

\usepackage{multirow}
\usepackage{wrapfig}

\usepackage{pifont}

\usepackage{MnSymbol}
\DeclareMathAlphabet\mathbb{U}{msb}{m}{n}
\usepackage{xpatch}

\def\Nset{\mathbb{N}}
\def\Rset{\mathbb{R}}

\def\Hset{\mathbb{H}}

\let\Pr\undefined

\DeclareMathOperator*{\Pr}{\mathbb{P}}

\DeclareMathOperator*{\E}{\mathbb E}
\DeclareMathOperator*{\argmax}{argmax}

\DeclareMathOperator{\sign}{sign}

\DeclarePairedDelimiter{\abs}{\lvert}{\rvert} 
\DeclarePairedDelimiter{\bracket}{[}{]}
\DeclarePairedDelimiter{\curl}{\{}{\}}
\DeclarePairedDelimiter{\paren}{(}{)}
\DeclarePairedDelimiter{\norm}{\|}{\|}

\newcommand{\sD}{{\mathscr D}}

\newcommand{\sG}{{\mathscr G}}
\newcommand{\sH}{{\mathscr H}}

\newcommand{\sM}{{\mathscr M}}

\newcommand{\sR}{{\mathscr R}}

\newcommand{\sX}{{\mathscr X}}
\newcommand{\sY}{{\mathscr Y}}
\newcommand{\sZ}{{\mathscr Z}}

\newcommand{\sfp}{{\mathsf p}}
\newcommand{\sfq}{{\mathsf q}}
\newcommand{\sfr}{{\mathsf r}}

\newcommand{\sfD}{{\mathsf D}}

\newcommand{\sfL}{{\mathsf L}}

\newcommand{\1}{1}

\newcommand{\Rad}{\mathfrak R}

\newcommand{\brho}{{\boldsymbol \rho}}

\newcommand{\hh}{{\mathsf h}}

\newcommand{\IMMAX}{\textsc{immax}}
\newcommand{\FOCAL}{\textsc{focal}}
\newcommand{\LDAM}{\textsc{ldam}}
\newcommand{\CE}{\textsc{ce}}

\newcommand{\CB}{\textsc{cb}}
\newcommand{\BS}{\textsc{bs}}
\newcommand{\EQUAL}{\textsc{equal}}
\newcommand{\RW}{\textsc{rw}}
\newcommand{\LA}{\textsc{la}}
\newcommand{\hinge}{\textsc{hinge}}

\newcommand{\h}{\widehat}
\newcommand{\ov}{\overline}

\newcommand{\wt}{\widetilde}
\newcommand{\e}{\epsilon}
\newcommand{\ignore}[1]{}

\theoremstyle{plain}
\newtheorem{theorem}{Theorem}[section]

\newtheorem{lemma}[theorem]{Lemma}
\newtheorem{corollary}[theorem]{Corollary}
\theoremstyle{definition}
\newtheorem{definition}[theorem]{Definition}

\theoremstyle{remark}

\usepackage[disable,textsize=tiny]{todonotes}

\hypersetup{
  breaklinks   = true, 
  colorlinks   = true, 
  urlcolor     = blue, 
  linkcolor    = blue, 
  citecolor   = blue 
}

\usepackage[toc, page, header]{appendix}
\icmltitlerunning{Learning from Imbalanced Data}

\begin{document}

\twocolumn[
\icmltitle{Balancing the Scales: A Theoretical and Algorithmic Framework for\\ Learning from Imbalanced Data}

\begin{icmlauthorlist}
\icmlauthor{Corinna Cortes}{google}
\icmlauthor{Anqi Mao}{courant}
\icmlauthor{Mehryar Mohri}{google,courant}
\icmlauthor{Yutao Zhong}{google}
\end{icmlauthorlist}

\icmlaffiliation{google}{Google Research, New York, NY;}
\icmlaffiliation{courant}{Courant Institute of Mathematical Sciences, New York, NY}

\icmlcorrespondingauthor{Corinna Cortes}{corinna@google.com}
\icmlcorrespondingauthor{Anqi Mao}{aqmao@cims.nyu.edu}
\icmlcorrespondingauthor{Mehryar Mohri}{mohri@google.com}
\icmlcorrespondingauthor{Yutao Zhong}{yutaozhong@google.com}

\icmlkeywords{imbalanced data, consistency, margin bounds, learning theory}

\vskip 0.3in
]

\printAffiliationsAndNotice{}

\begin{abstract}

Class imbalance remains a major challenge in machine learning,
especially in multi-class problems with long-tailed distributions.
Existing methods, such as data resampling, cost-sensitive techniques,
and logistic loss modifications, though popular and often effective, lack solid theoretical foundations. As an example, we demonstrate that cost-sensitive methods  are not Bayes-consistent.
This paper introduces a novel theoretical framework for analyzing
generalization in imbalanced classification.
\ignore{We first demonstrate the
Bayes-inconsistency of cost-sensitive methods and show that a previous
approach based on balanced loss contradicts our theoretical findings
for the standard misclassification loss.} We \ignore{then} propose a new
class-imbalanced margin loss function for both binary and multi-class
settings, prove its strong $\sH$-consistency, and derive corresponding
learning guarantees based on empirical loss and a new notion of
class-sensitive Rademacher complexity. Leveraging these theoretical
results, we devise novel and general learning algorithms,
\IMMAX\ (\emph{Imbalanced Margin Maximization}), which incorporate
confidence margins and are applicable to various hypothesis
sets. While our focus is theoretical, we also present extensive
empirical results demonstrating the effectiveness of our algorithms
compared to existing baselines.

\end{abstract}

\section{Introduction}
\label{sec:introduction}

The class imbalance problem, defined by a significant disparity in the
number of instances across classes within a dataset, is a common
challenge in machine learning applications
\citep{Lewis:1994,Fawcett:1996,KubatMatwin1997,kang2021exploring,
  menonlong, liu2019large,cui2019class}.  This issue is prevalent
in many real-world binary classification scenarios, and arguably even
more so in multi-class problems with numerous classes. In such cases,
a few majority classes often dominate the dataset, leading to a
``long-tailed'' distribution.  Classifiers trained on these imbalanced
datasets often struggle on the minority classes, performing similarly to a naive baseline that
simply predicts the majority class.

The problem has been widely studied in the literature
\citep{CardieNowe1997,KubatMatwin1997,chawla2002smote,he2009learning,
  WallaceSmallBrodleyTrikalinos2011}.  While a comprehensive review is
beyond our scope, we summarize key strategies into broad categories
and refer readers to a recent survey by
\citet{ZhangKangHooiYanFeng2023} for further details.
The primary approaches include the following.

\textbf{Data modification methods.} Techniques such as oversampling the minority classes \citep{chawla2002smote}, undersampling the majority
classes \citep{WallaceSmallBrodleyTrikalinos2011,KubatMatwin1997}, or
generating synthetic samples (e.g., SMOTE
\citep{chawla2002smote,QiaoLiu2008,han2005borderline}), aim to
rebalance the dataset before training
\citep{chawla2002smote,estabrooks2004multiple,
  liu2008exploratory,zhang2021learning}.

\textbf{Cost-sensitive techniques.} These assign different
penalization costs to losses for different classes. They include
cost-sensitive SVM \citep{Iranmehr:2019,Masnadi-Shirazi:2010} and
other cost-sensitive methods
\citep{elkan2001foundations,zhou2005training, zhao2018adaptive,
  zhang2018online, zhang2019online,sun2007cost,Fan:2017,
  jamal2020rethinking}. The weights are often determined by the
relative number of samples in each class or a notion of effective
sample size \cite{cui2019class}.

These two approaches are closely related and can be equivalent
in the limit, with cost-sensitive methods offering a more efficient
and principled implementation of data sampling. However, both
approaches act by effectively modifying the underlying distribution
and risk overfitting minority classes, discarding majority class
information, and inherently biasing the training distribution. Very
importantly, these techniques may lead to Bayes-inconsistency (proven
in Section~\ref{sec:existing-methods}). So while effective in some
cases, their performance depends on the problem, data distribution,
predictors, and evaluation metrics \citep{VanHulse:2007}, and they
often require extensive hyperparameter tuning. Hybrid approaches aim
to combine these two techniques but inherit many of their
limitations.

\textbf{Logistic loss modifications.} Several recent methods modify
the logistic loss to address class imbalance.
Some add hyperparameters to logits, effectively
implementing cost-sensitive adjustments to the loss's exponential
terms.  Examples include the Balanced Softmax loss
\citep{jiawei2020balanced}, Equalization loss
\citep{tan2020equalization}, and LDAM loss \citep{cao2019learning}.
Other methods, such as logit adjustment
\citep{menonlong,khan2019striking}, use hyperparameters for each
pair of class labels, with \citet{menonlong} showing
calibration for their approach.
\ignore{Alternative multiplicative modifications were advocated by
\citet{Ye:2020}. The Vector-Scaling loss \citep{kini2021label}
combines additive and multiplicative modifications.  The authors
present an analysis of this method in the case of linear predictors,
underscoring the specific benefits of the multiplicative changes.
Such multiplicative changes are
equivalent to normalizing scoring functions (or feature vectors in the
case of linear predictors), which is a common method irrespective
of the imbalanced data context.}
Alternative multiplicative modifications were advocated by
\citet{Ye:2020}, while the Vector-Scaling loss \citep{kini2021label}
integrates both additive and multiplicative adjustments. The authors
analyze this approach for linear predictors, highlighting the specific
advantages of multiplicative modifications. These multiplicative
adjustments, however, are equivalent to normalizing scoring functions
or feature vectors in linear cases, a widely used technique, regardless
of class imbalance.

\textbf{Other methods.} Additional approaches for addressing
imbalanced data (see \citep{ZhangKangHooiYanFeng2023}) include
post-hoc adjustments of decision thresholds
\citep{Fawcett:1996,Collell:2016} or class weights
\citep{kang2019decoupling,Kim:2019}, and techniques like transfer
learning, data augmentation, and distillation \citep{LiLiYeZhang2024}.

Despite the many significant advances, these techniques continue to face
persistent challenges.
Most existing solutions are heuristic-driven and lack a solid
theoretical foundation, making their performance unpredictable across
diverse contexts. Among prior work, the most closely related is that of \cite{cao2019learning},
which provides an analysis of generalization guarantees for the \emph{balanced loss} \citep{cortes2025improved}, 
which equalizes the impact of each class by weighting errors inversely to class frequency. Their analysis also applies only to binary
classification under the separable case and does not address the 
target \emph{misclassification loss}. \citet{jiawei2020balanced} adopt classical margin theory to derive generalization bounds for multi-class softmax regression, while \citet{wang2023unified} establish fine-grained generalization bounds based on data-dependent contraction. Despite these advances, all of these works focus exclusively on the balanced loss. In contrast, our work establishes generalization guarantees with respect to the standard zero-one misclassification loss.

\textbf{Loss functions and fairness considerations.} This work focuses
on the standard zero-one misclassification loss, which remains the
primary objective in many machine learning applications.
While the balanced loss is sometimes advocated for fairness,
particularly when labels correlate with demographic attributes, such
correlations are absent in many tasks. Moreover, fairness
involves broader considerations, and selecting the appropriate
criterion requires complex trade-offs.
Evaluation metrics like F1-score and AUC are also widely used in the
context of imbalanced data. However, these metrics can obscure the
model's performance on the standard zero-one misclassification tasks, especially in
scenarios with extreme imbalances or when the minority class exhibits
high variability.

\textbf{Our contributions.} This paper presents a comprehensive
theoretical analysis of generalization for classification loss in the
context of imbalanced classes.

In Section~\ref{sec:binary}, we introduce a \emph{class-imbalanced
margin loss function} and provide a novel theoretical analysis for
binary classification. We establish strong $\sH$-consistency bounds
and derive learning guarantees based on empirical class-imbalanced
margin loss and class-sensitive Rademacher complexity.
Section~\ref{sec:algorithms} details new learning algorithms,
\IMMAX\ (\emph{Imbalanced Margin Maximization}), inspired by our
theoretical insights. These algorithms generalize margin-based methods
by incorporating both positive and negative \emph{confidence
margins}. In the special case where the logistic loss is used, our
algorithms can be viewed as a logistic loss modification
method. However, they differ from previous approaches, including
multiplicative logit modifications, as our parameters are applied
multiplicatively to differences of logits, which naturally aligns with
the concept of margins.

In Section~\ref{sec:multiclass}, we extend our results to multi-class
classification, introducing a generalized multi-class class-imbalanced
margin loss, proving its $\sH$-consistency, and deriving
generalization bounds via confidence margin-weighted class-sensitive
Rademacher complexity. We also present new \IMMAX\ algorithms for
imbalanced multi-class problems based on these guarantees.
In Section~\ref{sec:existing-methods}, we analyze two core methods for
addressing imbalanced data.  We prove that cost-sensitive methods lack
Bayes-consistency and show that the analysis of
\citet{cao2019learning} in the separable binary case (for the balanced
loss) leads to margin values conflicting with our theoretical results
(for the misclassification loss). \ignore{Their approach also leads to
inferior empirical results.}
Finally, while the focus of our work is theoretical and algorithmic,
Section~\ref{sec:experiments} includes extensive empirical
evaluations, comparing our methods against several baselines.

\section{Preliminaries}
\label{sec:preliminaries}

\textbf{Binary classification.} Let $\sX$ represent the input space,
and $\sY = \curl*{-1, +1}$ the binary label space. Let $\sD$ be a
distribution over $\sX \times \sY$, and $\sH$ a hypothesis set of
functions mapping from $\sX$ to $\Rset$. Denote by
$\sH_{\mathrm{all}}$ the set of all measurable functions, and by $\ell
\colon \sH_{\mathrm{all}} \times \sX \times \sY \to \Rset$ a loss
function. The \emph{generalization error} of a hypothesis $h \in \sH$
and the \emph{best-in-class generalization error} of $\sH$ for a loss
function $\ell$ are defined as follows: $\sR_{\ell}(h) = \E_{(x,y)\sim
  \sD}\bracket*{\ell(h,x,y)}$, and $\sR_{\ell}^*(\sH) = \inf_{h \in
  \sH} \sR_{\ell}(h)$. The target loss function in binary
classification is the zero-one loss function defined for all $h \in
\sH$ and $(x, y) \in \sX \times \sY$ by $\ell_{0-1}(h, x, y) \coloneqq
\1_{\sign(h(x)) \neq y}$, where $\sign(\alpha) = \1_{\alpha \geq 0} -
\1_{\alpha < 0}$. For a labeled example $(x, y) \in \sX \times \sY$,
the \emph{margin} $\rho_{h}(x, y)$ of a predictor $h \in \sH$ is
defined by $\rho_h(x, y) = y h(x)$.  \ignore{ We denote the $\ell_p$
  norm by $\norm*{\cdot}_p$ with $p \in [1, \infty)$ and, for
    convenience, often use $\norm*{\cdot}$ to represent the case $p =
    2$.  }

\textbf{Consistency.} A fundamental property of a surrogate
loss $\ell_A$ for a target loss function $\ell_B$ is its
\emph{Bayes-consistency} \citep{Zhang2003,bartlett2006convexity,chen2004support}.  Specifically, if a sequence of predictors
$\{h_n\}_{n\in \Nset}\subset \sH_{\rm{all}}$ achieves the optimal
$\ell_A$-loss asymptotically, then it also achieves the optimal
$\ell_B$-loss in the limit: $\lim_{n \to +\infty} \sR_{\ell_A}(h_n) =
\sR^*_{\ell_A}(\sH_{\rm{all}}) \Rightarrow \lim_{n \to +\infty}
\sR_{\ell_B}(h_n) = \sR^*_{\ell_B}(\sH_{\rm{all}})$. While
Bayes-consistency is a natural and desirable property, it is
inherently asymptotic and applies only to the family of all measurable
functions $\sH_{\rm{all}}$.  A more applicable and informative notion
is that of \emph{$\sH$-consistent bounds}, which account for the
specific hypothesis class $\sH$ and provide non-asymptotic guarantees
\citep{awasthi2022h,awasthi2022multi,awasthi2021calibration,awasthi2021finer,AwasthiMaoMohriZhong2023theoretically,awasthi2023dc,mao2023cross,MaoMohriZhong2023ranking,MaoMohriZhong2023rankingabs,MaoMohriZhong2023structured,MaoMohriMohriZhong2023twostage,MaoMohriZhong2023characterization,zheng2023revisiting,MaoMohriZhong2024score,MaoMohriZhong2024predictor,MaoMohriZhong2024deferral,mao2024h,mao2024regression,mao2024multi,mao2024realizable,MohriAndorChoiCollinsMaoZhong2024learning,cortes2024cardinality,MaoMohriZhong2025mastering,MaoMohriZhong2025principled,mao2025theory,montreuil2024learning,montreuil2025two,montreuil2025adversarial,montreuil2025ask,montreuil2025one,desalvo2025budgeted,zhong2025fundamental}). In the realizable
setting, these bounds are of the form:
\begin{align*}
\forall h \in \sH, \quad & \sR_{\ell_B}(h) - \sR^*_{\ell_B}(\sH)
\leq \Gamma \paren*{\sR_{\ell_A}(h) - \sR^*_{\ell_A}(\sH)},
\end{align*}
where $\Gamma$ is a non-increasing concave function with $\Gamma(0) =
0$. In the general non-realizable setting, each side of the bound is
augmented with a \emph{minimizabily gap} $\sM_{\ell}(\sH) =
\sR_{\ell}^*(\sH) - \E_{x} \bracket*{ \inf_{h \in \sH} \E_{y}
  \bracket*{\ell(h, x, y) \mid x} }$, which measures the difference
between the best-in-class error and the expected best-in-class
conditional error.  The resulting bound is: $\sR_{\ell_B}(h) -
\sR^*_{\ell_B}(\sH) + \sM_{\ell_B}(\sH) \leq \Gamma
\paren*{\sR_{\ell_A}(h) - \sR^*_{\ell_A}(\sH) + \sM_{\ell_A}(\sH)}$.
$\sH$-consistency bounds imply Bayes-consistency when $\sH =
\sH_{\rm{all}}$ \citep{mao2024universal,mao2025enhanced,mohri2025beyond} and provide stronger and more
applicable guarantees.

\section{Theoretical Analysis of Imbalanced Binary Classification}
\label{sec:binary}

Our theoretical analysis addresses imbalance by introducing distinct
\emph{confidence margins} for positive and negative points. This
allows us to explicitly account for the effects of class imbalance.
We begin by defining a general class-imbalanced margin loss function
based on these confidence margins. Subsequently, we prove that, unlike
previously studied cost-sensitive loss functions in the literature,
this new loss function satisfies $\sH$-consistency
bounds. Furthermore, we establish general margin bounds for imbalanced
binary classification in terms of the proposed class-imbalanced margin
loss.
While our use of margins bears some resemblance to the approach of \citet{cao2019learning}, their analysis is limited to
\emph{geometric margins} in the separable case, making ours
fundamentally distinct.

\subsection{Imbalanced $(\rho_{+}, \rho_{-})$-Margin Loss Function}
\label{sec:imbalanced-loss}

We first extend the $\rho$-margin loss function
\citep{MohriRostamizadehTalwalkar2018} to accommodate the imbalanced
setting. To account for different confidence margins for instances
with label $+$ and label $-$, we define the \emph{class-imbalanced
$(\rho_{+}, \rho_{-})$-margin loss function} as follows:
\begin{definition}[Class-imbalanced margin loss function]
\label{def:imbalanced-loss}
Let $\Phi_{\rho} \colon u \mapsto \min \paren*{1, \max \paren*{0, 1 -
    \frac{u}{\rho}}}$ be the $\rho$-margin loss function. For any $\rho_{+} > 0$ and $\rho_{-} > 0$, the \emph{class-imbalanced
$(\rho_{+}, \rho_{-})$-margin loss} is the function $\sfL_{\rho_{+},
  \rho_{-}} \colon \sH_{\mathrm{all}} \times \sX \times \sY \to
\Rset$, defined as follows:
\begin{equation*}
  \sfL_{\rho_{+}, \rho_{-}}(h, x, y) = \Phi_{\rho_{+}}(y h(x)) \1_{y  = +1}
  + \Phi_{\rho_{-}}(y h(x)) \1_{y  = -1}.
\end{equation*}
\ignore{where $\Phi_{\rho} \colon u \mapsto \min \paren*{1, \max \paren*{0, 1 -
    \frac{u}{\rho}}}$ is the $\rho$-margin loss function.}
\end{definition}
The main margin bounds in this section are expressed in terms of this
loss function. The parameters $\rho_{+}$ and $\rho_{-}$, both greater
than 0, represent the confidence margins imposed by a hypothesis $h$
for positive and negative instances, respectively.  The following
result provides an equivalent expression for the class-imbalanced margin
loss function, see proof in Appendix~\ref{app:margin-loss}.

\begin{restatable}{lemma}{MarginLoss}
\label{lemma:margin-loss}
The class-imbalanced $(\rho_{+}, \rho_{-})$-margin loss function can be
equivalently expressed as follows:
\begin{equation*}
  \sfL_{\rho_{+}, \rho_{-}}\!(h, x, y) \!=\! \Phi_{\rho_{+}}\!(y h(x)) \1_{h(x) \geq 0}
  \!+\! \Phi_{\rho_{-}}\!(y h(x)) \1_{h(x) < 0}. \mspace{-8mu}
\end{equation*}
\end{restatable}
\ignore{The proof is included in Appendix~\ref{app:margin-loss}. }\ignore{ The empirical
class-imbalanced margin loss is similarly defined as the imbalanced margin
loss computed over the training sample.}

\subsection{$\sH$-Consistency}
\label{sec:H-consistency}

The following result provides a strong consistency guarantee for the
class-imbalanced margin loss introduced in relation to the zero-one loss. We
say a hypothesis set is complete when the scoring values spanned by
$\sH$ for each instance cover $\Rset$: for all $x \in \sX$,
$\curl*{h(x) \colon h \in \sH} = \Rset$. Most hypothesis sets widely
considered in practice are all complete.

\begin{restatable}[$\sH$-consistency bound for class-imbalanced margin loss]
  {theorem}
  {HConsistencyBinary}
\label{thm:H-consistency-binary}
Let $\sH$ be a complete hypothesis set. Then, for all $h \in \sH$,
$\rho_{+} > 0$, and $\rho_{-} > 0$, the following bound holds:
\ifdim\columnwidth=\textwidth
{
\begin{equation}
  \sR_{\ell_{0-1}}(h) - \sR^*_{\ell_{0-1}}(\sH) + \sM_{\ell_{0-1}}(\sH)
  \leq \sR_{\sfL_{\rho_{+}, \rho_{-}}}(h) - \sR^*_{\sfL_{\rho_{+}, \rho_{-}}}(\sH)
  + \sM_{\sfL_{\rho_{+}, \rho_{-}}}(\sH).
\end{equation}
}\else
{
\begin{multline*}
\sR_{\ell_{0-1}}(h) - \sR^*_{\ell_{0-1}}(\sH) + \sM_{\ell_{0-1}}(\sH)\\
\leq \sR_{\sfL_{\rho_{+}, \rho_{-}}}(h) - \sR^*_{\sfL_{\rho_{+}, \rho_{-}}}(\sH)
  + \sM_{\sfL_{\rho_{+}, \rho_{-}}}(\sH).
\end{multline*}
}\fi
\end{restatable}

The proof is presented in Appendix~\ref{app:H-consistency-binary}. Note that our $\sH$-consistency bounds in Theorem~\ref{thm:H-consistency-binary} can be extended directly to the uniformly bounded hypothesis sets considered in Theorem~\ref{thm:rad-linear} below. In this case, the bounds would depend on the complexity of the hypothesis class, similar to the $\sH$-consistency bounds presented in \citep{awasthi2022h}.
The next section presents generalization bounds based on the empirical
class-imbalanced margin loss, along with the \emph{$(\rho_{+},
\rho_{-})$-class-sensitive Rademacher complexity} and its empirical
counterpart defined below. Given a sample $S = \paren*{x_1, \ldots,
  x_m}$, we define $I_+ = \curl*{i \in \curl*{1, \ldots, m} \mid y_i =
  +1}$ and $m_{+} = |I_+|$ as the number of positive
instances. Similarly, we define $I_{-} = \curl*{i \in \curl*{1,
    \ldots, m} \mid y_i = -1}$ and $m_{-} = |I_{-}|$ as the number of
negative instances.

\begin{definition}[$(\rho_{+}, \rho_{-})$--class-sensitive Rademacher complexity]
Let $\sG$ be a family of functions mapping from $\sZ$ to $[a, b]$ and
$S = \paren*{z_1, \ldots, z_m}$ a fixed sample of size $m$ with
elements in $\sZ$. Fix $\rho_{+} > 0$ and $\rho_{-} > 0$. Then, the
\emph{empirical $(\rho_{+}, \rho_{-})$-class-sensitive
Rademacher complexity of $\sG$} with respect to the sample $S$ is
defined as:
\[
\h \Rad_{S}^{\rho_{+}, \rho_{-}}(\sG)
= \frac{1}{m} \E_{\sigma} \bracket*{\sup_{g \in \sG}
  \curl*{ \sum_{i \in I_+} \frac{\sigma_i g(z_i)}{\rho_{+}}
    + \sum_{i \in I_{-}}  \frac{\sigma_i g(z_i)}{\rho_{-}} }},
\]
where $\sigma = \paren*{\sigma_1, \ldots, \sigma_m}^{\top}$, with
$\sigma_i$s independent uniform random variables taking values in
$\curl*{-1, +1}$.
For any integer $m \geq 1$, the \emph{$(\rho_{+},
\rho_{-})$-class-sensitive Rademacher complexity of $\sG$} is the
expectation of the empirical $(\rho_{+},
\rho_{-})$--class-sensitive Rademacher complexity over all samples of
size $m$ drawn according to $\sD$: $\Rad_{m}^{\rho_{+}, \rho_{-}}(\sG)
= \E_{S \sim \sD^m} \bracket*{\h \Rad_{S}^{\rho_{+}, \rho_{-}}(\sG)}$.
\end{definition}

\subsection{Margin-Based Guarantees}
\label{sec:margin-bound}

Next, we will prove a general margin-based generalization bound, which
will serve as the foundation for deriving new algorithms for
imbalanced binary classification.

Given a sample $S = \paren*{x_1, \ldots, x_m}$ and a hypothesis $h$,
the \emph{empirical class-imbalanced margin loss} is defined by $\h
\sR_{S}^{\rho_{+}, \rho_{-}}(h) = \frac{1}{m} \sum_{i = 1}^m
\sfL_{\rho_{+}, \rho_{-}}(h, x_i, y_i)$.  Note that the zero-one loss
function $\ell_{0-1}$ is upper-bounded by the
class-imbalanced margin loss function $\sfL_{\rho_{+}, \rho_{-}}$:
$\sR_{\ell_{0-1}}(h) \leq \sR_{\sfL_{\rho_{+}, \rho_{-}}}(h)$.

\begin{restatable}[Margin bound for imbalanced binary classification]{theorem}
  {MarginBoundBinary}
\label{thm:margin-bound-binary}
Let $\sH$ be a set of real-valued functions. Fix $\rho_{+} > 0$ and
$\rho_{-} > 0$, then, for any $\delta > 0$, with probability at least
$1 - \delta$, each of the following holds for all $h \in \sH$:
\begin{align*}
  \sR_{\ell_{0-1}}(h) &\leq \h \sR_{S}^{\rho_{+}, \rho_{-}}(h)
  + 2 \Rad_{m}^{\rho_{+}, \rho_{-}}(\sH) + \sqrt{\frac{\log \frac{1}{\delta}}{2m}}\\
  \sR_{\ell_{0-1}}(h) &\leq \h \sR_{S}^{\rho_{+}, \rho_{-}}(h)
  + 2 \h \Rad_{S}^{\rho_{+}, \rho_{-}}(\sH) + 3 \sqrt{\frac{\log \frac{2}{\delta}}{2m}}.
\end{align*}
\end{restatable}
The proof is presented in Appendix~\ref{app:margin-bound-binary}. The
generalization bounds in Theorem~\ref{thm:margin-bound-binary} suggest
a trade-off: increasing $\rho_{+}$ and $\rho_{-}$ reduces the
complexity term (second term) but increases the empirical class-imbalanced
margin loss $\h \sR_{S}^{\rho_{+}, \rho_{-}}(h)$ (first term) by
requiring higher confidence margins from the hypothesis
$h$. Therefore, if the empirical class-imbalanced margin loss of $h$ remains
small for relatively large values of $\rho_{+}$ and $\rho_{-}$, $h$
admits a particularly favorable guarantee on its generalization error.

For Theorem~\ref{thm:margin-bound-binary}, the margin parameters
$\rho_{+}$ and $\rho_{-}$ must be selected beforehand. But, the bounds
of the theorem can be generalized to hold uniformly for all $\rho_{+}
\in (0, 1]$ and $\rho_{-} \in (0, 1]$ at the cost of modest additional
  terms $\sqrt{\frac{\log \log_2 \frac{2}{\rho_{+}}}{m}}$ and
  $\sqrt{\frac{\log \log_2 \frac{2}{\rho_{-}}}{m}}$, as shown in
  Theorem~\ref{thm:margin-bound-binary-uniform} in
  Appendix~\ref{app:margin-bound-binary-uniform}.

\section{Algorithm for Binary Classification}
\label{sec:algorithms}

In this section, we derive algorithms for binary classification in
imbalanced settings, building on the theoretical analysis from the
previous section.

\textbf{Explicit guarantees.}
Let $S \subseteq \curl*{x \colon \norm*{x} \leq r}$ denote a sample of
size $m$. Define $r_{+} = \sup_{i \in I_+}
\norm*{x_i}$ and $r_{-} = \sup_{i \in I_{-}} \norm*{x_i}$. We assume
that the empirical class-sensitive Rademacher complexity $\h
\Rad_{S}^{\rho_{+}, \rho_{-}}(\sH)$ can be bounded as:
\begin{align*}
  \h \Rad_{S}^{\rho_{+}, \rho_{-}}(\sH)
  & \leq \frac{\Lambda_{\sH}}{m} \sqrt{\frac{m_{+} r_{+}^2}{\rho_{+}^2}
    + \frac{m_{-} r_{-}^2}{\rho_{-}^2}}
  \leq \frac{\Lambda_{\sH} r}{m} \sqrt{\frac{m_{+}}{\rho_{+}^2}
    + \frac{m_{-}}{\rho_{-}^2}}, \mspace{-6mu}
\end{align*}
where $\Lambda_{\sH}$ depends on the complexity of the hypothesis set
$\sH$. This bound holds for many commonly used hypothesis sets.  As an
example, for a family of neural networks, $\Lambda_\sH$ can be
expressed as a Frobenius norm
\citep{CortesGonzalvoKuznetsovMohriYang2017,NeyshaburTomiokaSrebro2015}
or spectral norm complexity with respect to reference weight matrices
\cite{BartlettFosterTelgarsky2017}.  More generally, for the analysis
that follows, we will assume that $\sH$ can be defined by $\sH =
\curl*{h \in \ov \sH \colon \norm{h} \leq \Lambda_\sH}$, for some
appropriate norm $\norm*{\, \cdot \,}$ on some space $\ov \sH$.  For
the class of linear hypotheses with bounded weight vector, $\sH =
\curl*{x \mapsto w \cdot x \colon \norm*{w} \leq \Lambda}$, we provide
the following explicit guarantee. The proof is presented in
Appendix~\ref{app:rad-linear}.

\begin{restatable}{theorem}{RadLinear}
\label{thm:rad-linear}
Let $S \subseteq \curl*{x \colon \norm*{x} \leq r}$ be a sample of
size $m$ and let $\sH = \curl*{x \mapsto w \cdot x \colon \norm*{w}
  \leq \Lambda}$. Let $r_{+} = \sup_{i \in I_+} \norm*{x_i}$ and
$r_{-} = \sup_{i \in I_{-}} \norm*{x_i}$. Then, the following bound
holds for all $h \in \sH$:
\begin{equation*}
  \h \Rad_{S}^{\rho_{+}, \rho_{-}}(\sH)
  \leq \frac{\Lambda}{m} \sqrt{\frac{m_{+} r_{+}^2}{\rho_{+}^2}
    + \frac{m_{-} r_{-}^2}{\rho_{-}^2}}
  \leq \frac{\Lambda r}{m} \sqrt{\frac{m_{+}}{\rho_{+}^2}
    + \frac{m_{-}}{\rho_{-}^2}}.
\end{equation*}
\end{restatable}
Combining the upper bound of Theorem~\ref{thm:rad-linear}
and Theorem~\ref{thm:margin-bound-binary} gives directly the following
general margin bound:
\begin{align*}
\mspace{-10mu} \sR_{\ell_{0-1}}\mspace{-3mu} (h) 
& \mspace{-3mu} \leq \mspace{-3mu} \h \sR_{S}^{\rho_{+}, \rho_{-}}\mspace{-3mu}  (h)
\mspace{-3mu} + \mspace{-3mu} \frac{2 \Lambda_\sH}m \sqrt{\mspace{-3mu} \frac{m_{+} r_{+}^2}{\rho_{+}^2} 
\mspace{-3mu} + \mspace{-3mu} \frac{m_{-} r_{-}^2}{\rho_{-}^2}}
\mspace{-3mu} + \mspace{-3mu} 3 \sqrt{\frac{\log \frac{2}{\delta}}{2m}}. \mspace{-10mu}
\end{align*}
As with Theorem~\ref{thm:margin-bound-binary}, this bound
can be generalized to hold uniformly for all $\rho_{+} \in
(0, 1]$ and $\rho_{-} \in (0, 1]$ at the cost of additional terms
    $\sqrt{\frac{\log \log_2 \frac{2}{\rho_{+}}}{m}}$ and
    $\sqrt{\frac{\log \log_2 \frac{2}{\rho_{-}}}{m}}$ by combining
    the bound on the class-sensitive Rademacher complexity and
    Theorem~\ref{thm:margin-bound-binary-uniform}.
The bound suggests that a small generalization error can be achieved
when the second term $\frac{\Lambda_\sH}m \sqrt{\frac{m_{+}
    r_{+}^2}{\rho_{+}^2} + \frac{m_{-} r_{-}^2}{\rho_{-}^2}} $ or
$\frac{\Lambda_\sH r}{m} \sqrt{\frac{m_{+}}{\rho_{+}^2} +
  \frac{m_{-}}{\rho_{-}^2}}$ is small while the empirical class-imbalanced
margin loss (first term) remains low.

Now, consider a margin-based loss function $(h, x, y) \mapsto \Psi(y h(x))$
defined using a non-increasing convex function $\Psi$ such that
$\Phi_{\rho}(u) \leq \Psi\left(\frac{u}{\rho}\right)$ for all $u \in
\Rset$. Examples of such $\Psi$ include: the hinge loss, $\Psi(u) =
\max(0, 1 - u)$, the logistic loss, $\Psi(u) = \log_2(1 + e^{-u})$,
and the exponential loss, $\Psi(u) = e^{-u}$.

Then, choosing $\Lambda_\sH = 1$, with probability at least $1 -
\delta$, the following holds for all $h \in \curl*{h \in \ov \sH
  \colon \norm{h} \leq 1}$, $\rho_{+} \in (0, r_{+}]$ and $\rho_{-}
  \in (0, r_{-}]$:
\begin{align*}
\sR_{\ell_{0-1}}(h)
& \leq \frac{1}{m} 
\bracket*{ \sum_{i \in I_+} \Psi \paren*{\frac{y_i h(x_i)}{\rho_{+}}}
  + \sum_{i \in I_-}\Psi \paren*{\frac{y_i h(x_i)}{\rho_{-}}}}\\
& \quad + \frac{4 r}{m} \sqrt{\frac{m_{+}}{\rho_{+}^2} + \frac{m_{-}}{\rho_{-}^2}}
+ O\paren*{\frac{1}{\sqrt{m}}},
\end{align*}
where the last term includes the $\log$-$\log$ terms and the
$\delta$-confidence term.

Since for any $\rho > 0$, $h / \rho$ admits the same generalization
error as $h$, with probability at least $1 - \delta$, the following
holds for all $h \in \curl*{h \in \ov \sH \colon \norm*{h} \leq
  \frac{1}{\rho_{+} + \rho_{-}}}$, $\rho_{+} \ignore{\in (0, r_{+}]}$ and
  $\rho_{-} \ignore{\in (0, r_{-}]}$:
\begin{align*}
& \sR_{\ell_{0-1}}(h) 
\leq \frac{1}{m} \bracket[\bigg]{
  \sum_{i \in I_+} \Psi\paren*{y_i h(x_i) \frac{\rho_{+}
+ \rho_{-}}{\rho_{+}}} \\ 
& + \sum_{i \in I_-} \Psi\paren*{y_i h(x_i) \frac{\rho_{+}
+ \rho_{-}}{\rho_{-}}}}
+ \frac{4 r}{m} \sqrt{\frac{m_{+}}{\rho_{+}^2}
  + \frac{m_{-}}{\rho_{-}^2}}  + O\paren*{\frac{1}{\sqrt{m}}}. \mspace{-12.5mu}
\end{align*}
\textbf{Algorithm.}
Now, since only the first term of the right-hand side depends on $h$,
the bound suggests selecting $h$,
with $\norm*{h}^2 \leq \paren*{\frac{1}{\rho_{+} + \rho_{-}}}^2$ as a solution of:
\begin{align*}
  \min_{h \in \ov \sH}
\frac{1}{m} \bracket[\bigg]{
  & \sum_{i \in I_+} \Psi\paren*{y_i h(x_i) \tfrac{\rho_{+}
+ \rho_{-}}{\rho_{+}}}
   + \sum_{i \in I_-} \Psi\paren*{y_i h(x_i) \tfrac{\rho_{+}
+ \rho_{-}}{\rho_{-}}}}.
\end{align*}
Introducing a Lagrange multiplier $\lambda \geq 0$ and a free variable
$\alpha = \frac{\rho_{+}}{\rho_{+} + \rho_{-}} > 0$, the optimization
problem can be written as
\begin{equation}
\label{eqn:binary}
  \min_{h \in \ov \sH} \lambda \norm*{h}^2  +
  \frac{1}{m} \bracket*{
  \sum_{i \in I_+} \Psi\paren*{\frac{h(x_i)}{\alpha}}
  + \sum_{i \in I_-} \Psi\paren*{\frac{- h(x_i)}{1 - \alpha}}},   
\end{equation}
where $\lambda$ and $\alpha$ can be selected via cross-validation.

This formulation provides a general algorithm for binary
classification in imbalanced settings, called \textbf{\IMMAX\
(\emph{Imbalanced Margin Maximization})}, supported by strong
theoretical guarantees derived in the previous section. This provides
a solution for optimizing the decision boundaries in imbalanced
settings based on confidence margins. In the specific case of linear
hypotheses (Appendix~\ref{app:linear}), choosing $\Psi$ as the
Hinge loss yields a strict generalization of the SVM algorithm which
can be used with positive definite kernels, or a strict generalization
of the logistic regression algorithm when $\Psi$ defines the logistic
loss.

Beyond linear models, this algorithm readily extends to neural
networks with various regularization terms and other complex
hypothesis sets.  This makes it a general solution for tackling
imbalanced binary classification problems.

\textbf{Separable case}. When the training sample is separable, we can
denote by $\rho_{\rm{geom}}$ the geometric margin, that is the
smallest distance of a training sample point to the decision boundary
measured in the Euclidean distance or another metric appropriate for
the feature space. As an example, for linear hypotheses,
$\rho_{\rm{geom}}$ corresponds to the familiar Euclidean distance to
the separating hyperplane.

The confidence margin parameters $\rho_{+}$ and $\rho_{-}$ can then be
chosen so that $\rho_{+} + \rho_{-} = 2\rho_{\rm{geom}}$, ensuring
that the empirical class-imbalanced margin loss term is zero.
Minimizing the right-hand side of the bound then yields the following
expressions for $\rho_+$ and $\rho_-$:
\begin{align*}
& \rho_{+} = \frac{2m^{\frac{1}{3}}_{+}r_{+}^{\frac{2}{3}}}{m^{\frac{1}{3}}_{+}r_{+}^{\frac{2}{3}} + m^{\frac{1}{3}}_{-}r_{-}^{\frac{2}{3}}} \rho_{\rm{geom}}
& \rho_{-} = \frac{2m^{\frac{1}{3}}_{-} r_{-}^{\frac{2}{3}}}{m^{\frac{1}{3}}_{+}r_{+}^{\frac{2}{3}} + m^{\frac{1}{3}}_{-}r_{-}^{\frac{2}{3}}} \rho_{\rm{geom}}.
\end{align*}
For $r_+ = r_-$, these expressions simplify to:
\begin{align}
\label{eq:optimal-separable}
  & \rho_{+} = \frac{2m^{\frac{1}{3}}_{+}}{m^{\frac{1}{3}}_{+} + m^{\frac{1}{3}}_{-}}
  \rho_{\rm{geom}}
  & \rho_{-} = \frac{2m^{\frac{1}{3}}_{-}}{m^{\frac{1}{3}}_{+} + m^{\frac{1}{3}}_{-}}
  \rho_{\rm{geom}}.
\end{align}
Note that  the optimal positive margin $\rho_+$ is larger than the negative
one $\rho_-$ when there are more positive samples than negative ones
($m_+ > m_-$). Thus, in the linear case, this suggests selecting a
hyperplane with a large positive margin in that case, see
Figure~\ref{fig:sep} for an illustration.

Finally, note that, while $\alpha = \frac{\rho_{+}}{\rho_{+} + \rho_{-}} >0$ in the optimization problem (\ref{eqn:binary}) can be freely searched over a range
of values in our general (non-separable case) algorithm, it can be
beneficial to focus the search around the optimal values identified in
the separable case.

\section{Extension to Multi-Class Classification}
\label{sec:multiclass}

In this section, we extend our results to multi-class classification,
with full details provided in Appendix~\ref{app:multiclass}. Below, we
present a concise overview.

We will adopt the same notation and definitions as previously
described, with some slight adjustments. In particular, we denote the
multi-class label space by $\sY = [c] \coloneqq \curl*{1, \ldots, c}$
and a hypothesis set of functions mapping from $\sX \times \sY$ to
$\Rset$ by $\sH$. For a hypothesis $h \in \sH$, the label $\hh(x)$
assigned to $x \in \sX$ is the one with the largest score, defined as
$\hh(x) = \argmax_{y \in \sY} h(x, y)$, using the highest index for
tie-breaking. For a labeled example $(x, y) \in \sX \times \sY$, the
\emph{margin} $\rho_h(x, y)$ of a hypothesis $h \in \sH$ is given by
$\rho_h(x, y) = h(x, y) - \max_{y' \neq y} h(x, y')$, which is the
difference between the score assigned to $(x, y)$ and that of the
next-highest scoring label. We define the multi-class zero-one loss
function as $\ell^{\rm{multi}}_{0-1} \coloneqq \1_{\hh(x) \neq
  y}$. This is the target loss of interest in multi-class
classification.

We define the \emph{multi-class class-imbalanced
margin loss function} as follows:
\begin{definition}[Multi-class class-imbalanced margin loss]
For any $\brho = [\rho_k]_{k \in [c]}$, the \emph{multi-class
class-imbalanced $\brho$-margin loss} is the function $\sfL_{\brho} \colon
\sH_{\mathrm{all}} \times \sX \times \sY \to \Rset$, defined by:
\begin{equation}
  \sfL_{\brho}(h, x, y) = \sum_{k = 1}^c
  \Phi_{\rho_k}\paren*{\rho_h(x, y)} \1_{y  = k}.
\end{equation}
\end{definition}
The main margin bounds in this section are expressed in terms of this
loss function. The parameters $\rho_k > 0$, for $k \in [c]$, represent
the confidence margins imposed by a hypothesis $h$ for instances
labeled $k$. As in the binary case, we establish an equivalent
expression for this class-imbalanced margin loss function.
\begin{restatable}{lemma}{MarginLossMulti}
\label{lemma:margin-loss-multi}
The multi-class class-imbalanced $\brho$-margin loss can be equivalently
expressed as follows:
\begin{equation*}
  \sfL_{\brho}(h, x, y) = \sum_{k = 1}^c
  \Phi_{\rho_k}\paren*{\rho_h(x, y)} \1_{\hh(x)  = k}.
\end{equation*}
\end{restatable}
The proof is included in
Appendix~\ref{app:margin-loss-multi}.
We also prove that our multi-class class-imbalanced $\brho$-margin
loss is $\sH$-consistent for any \emph{complete} hypothesis set $\sH$. We say a hypothesis set is complete when
the scoring values spanned by $\sH$ for each instance cover $\Rset$:
for all $(x, y) \in \sX \times \sY$, $\curl*{h(x, y) \colon h \in \sH}
= \Rset$. This covers all commonly used
function classes in practice, such as linear classifiers and neural
network architectures.

\begin{restatable}[$\sH$-Consistency bound for multi-class class-imbalanced
    margin loss]{theorem}{HConsistencyMulti}
\label{thm:H-consistency-multi}
Let $\sH$ be a complete hypothesis set. Then, for all $h \in \sH$ and
$\brho = [\rho_k]_{k \in [c]} > \mathbf{0}$, the following bound
holds: \ifdim\columnwidth=\textwidth {
\begin{equation}
  \sR_{\ell^{\rm{multi}}_{0-1}}(h) - \sR^*_{\ell^{\rm{multi}}_{0-1}}(\sH)
  + \sM_{\ell^{\rm{multi}}_{0-1}}(\sH)
  \leq \sR_{\sfL_{\brho}}(h) - \sR^*_{\sfL_{\brho}}(\sH) + \sM_{\sfL_{\brho}}(\sH).
\end{equation}
}\else
{
\begin{multline*}
\sR_{\ell^{\rm{multi}}_{0-1}}(h) - \sR^*_{\ell^{\rm{multi}}_{0-1}}(\sH)
  + \sM_{\ell^{\rm{multi}}_{0-1}}(\sH)\\
\leq \sR_{\sfL_{\brho}}(h) - \sR^*_{\sfL_{\brho}}(\sH) + \sM_{\sfL_{\brho}}(\sH).
\end{multline*}
}\fi
\end{restatable}

The proof is included in Appendix~\ref{app:H-consistency-multi}.
Our generalization bounds are expressed in terms of the following
notions of $\brho$-class-sensitive Rademacher complexity.

\begin{definition}[$\brho$-class-sensitive Rademacher complexity]
Let $\sH$ be a family of functions mapping from $\sX \times \sY$ to
$\Rset$ and $S = \paren*{(x_1, y_1) \ldots, (x_m, y_m)}$ a fixed
sample of size $m$ with elements in $\sX \times \sY$. Fix $\brho =
[\rho_k]_{k \in [c]} > \mathbf{0}$. Then, the \emph{empirical
$\brho$-class-sensitive Rademacher complexity of $\sH$} with respect to
the sample $S$ is defined as:
\begin{equation}
  \h \Rad_{S}^{\brho}(\sH)
  = \frac{1}{m} \E_{\e}\bracket*{\sup_{h \in \sH}
    \curl*{ \sum_{k = 1}^c \sum_{i \in I_k}
      \sum_{y \in \sY} \e_{iy} \frac{h(x_i, y)}{\rho_{k}} }},
\end{equation}
where $\e = \paren*{\e_{i y}}_{i, y}$ with $\e_{iy}$s being
independent variables uniformly distributed over $\curl*{-1, +1}$.
For any integer $m \geq 1$, the \emph{$\brho$-class-sensitive
Rademacher complexity of $\sH$} is the expectation of the empirical
$\brho$-class-sensitive Rademacher complexity over all samples
of size $m$ drawn according to $\sD$: $\Rad_{m}^{\brho}(\sH) = \E_{S
  \sim \sD^m} \bracket*{\h \Rad_{S}^{\brho}(\sH)}$.
\end{definition}

\textbf{Margin bound.} We establish a general multi-class margin-based
generalization bound in terms of the empirical multi-class
class-imbalanced $\brho$-margin loss and the empirical
$\brho$-class-sensitive Rademacher complexity
(Theorem~\ref{thm:margin-bound-multi}). The bound takes the following
form:
\[
\sR_{\ell^{\rm{multi}}_{0-1}}(h)
\leq \h \sR_{S}^{\brho}(h)
  + 4 \sqrt{2c}\, \Rad_{m}^{\brho}(\sH) + O(1/\sqrt{m}).
\]
This serves as the foundation for deriving new algorithms for
imbalanced multi-class classification.

\textbf{Explicit guarantees.} Let $\Phi$ be a feature mapping from
$\sX \times \sY$ to $\Rset^{d}$. Let $S \subseteq \curl*{(x, y) \colon
  \norm*{\Phi(x, y)} \leq r}$ denote a sample of size $m$, for some
appropriate norm $\norm*{\, \cdot \,}$ on $\Rset^d$. Define $r_{k} =
\sup_{i \in I_k, y \in \sY} \norm*{\Phi(x_i, y)}$, for any $k \in
    [c]$. As in the binary case, we assume that the empirical
    class-sensitive Rademacher complexity $\h \Rad_{S}^{\brho}(\sH)$
    can be bounded as:
\begin{equation*}
  \h \Rad_{S}^{\brho}(\sH)
  \leq \frac{\Lambda_{\sH} \sqrt{c}}{m} \sqrt{\sum_{k = 1}^c \frac{m_k r_{k}^2}{\rho_k^2}}
  \leq \frac{\Lambda_{\sH} r \sqrt{c}}{m} \sqrt{\sum_{k = 1}^c \frac{m_k}{\rho_k^2}},
\end{equation*}
where $\Lambda_{\sH}$ depends on the complexity of the hypothesis set
$\sH$. This bound holds for many commonly used hypothesis sets.  For a
family of neural networks, $\Lambda_\sH$ can be expressed as a
Frobenius norm
\citep{CortesGonzalvoKuznetsovMohriYang2017,NeyshaburTomiokaSrebro2015}
or spectral norm complexity with respect to reference weight matrices
\cite{BartlettFosterTelgarsky2017}. Additionally,
Theorems~\ref{thm:rad-kernel-1} and \ref{thm:rad-kernel-2} in
Appendix~\ref{app:linear-multi} address kernel-based hypotheses. More
generally, for the analysis that follows, we will assume that $\sH$
can be defined by $\sH = \curl*{h \in \ov \sH \colon \norm{h} \leq
  \Lambda_\sH}$, for some appropriate norm $\norm*{\, \cdot \,}$ on
some space $\ov \sH$.  Combining such an upper bound and
Theorem~\ref{thm:margin-bound-multi} or
Theorem~\ref{thm:margin-bound-multi-uniform}, gives directly the
following general margin bound:
\begin{align*}
  \sR_{\ell^{\rm{multi}}_{0-1}}(h) &\leq \h \sR_{S}^{\brho}(h)
  + \frac{4 \sqrt{2} \Lambda_{\sH} r  c }{m}
  \sqrt{\sum_{k = 1}^c \frac{m_k}{\rho_k^2}} + O\paren*{\frac{1}{\sqrt{m}}},
\end{align*}
where the last term includes the $\log$-$\log$ terms and the
$\delta$-confidence term. Let $\Psi$ be a non-increasing convex
function such that $\Phi_{\rho}(u) \leq
\Psi\left(\frac{u}{\rho}\right)$ for all $u \in \Rset$. Then, since
$\Phi_\rho$ is non-increasing, for any $(x, k)$, we have:
$\Phi_{\rho}(\rho_h(x, k))
 = \max_{j \neq k} \Phi_{\rho}(h(x, k) - h(x, j)).$

\textbf{Algorithm.} This suggests a regularization-based algorithm of
the following form:
\begin{equation}
\min_{h \in \ov \sH} \lambda \norm*{h}^2
+ \frac{1}{m} \bracket*{\sum_{k = 1}^c
  \sum_{i \in I_k} \max_{j \neq k}
  \Psi \paren*{\tfrac{h(x, k) - h(x, j)}{\rho_k}}},
\end{equation}
where $\lambda$ and $\rho_k$s are chosen via cross-validation.  In
particular, choosing $\Psi$ to be the logistic loss and upper-bounding
the maximum by a sum yields the following form for our
\IMMAX\ (\emph{Imbalanced Margin Maximization}) algorithm:
\begin{equation}
\label{eq:alg-gen}
\min_{h \in \ov \sH} \lambda \norm*{h}^2
+ \frac{1}{m} \sum_{k = 1}^c
\sum_{i \in I_k} \mspace{-2mu} \log \bracket*{\sum_{j = 1}^c
  \exp\paren*{\tfrac{h(x_i, j) - h(x_i, k)}{\rho_{k}}}},
\end{equation}
where $\lambda$ and $\rho_k$s are chosen via cross-validation.
Let $\rho = \sum_{k = 1}^c \rho_k$ and $\ov r = \bracket*{\sum_{k =
    1}^c m_k^{\frac{1}{3}} r_{k, 2}^{\frac{2}{3}}}^{\frac{3}{2}}$, where as defined previously, $r_{k} =
\sup_{i \in I_k, y \in \sY} \norm*{\Phi(x_i, y)}$, for any $k \in
    [c]$.
Using Lemma~\ref{lemma:D3} (Appendix~\ref{app:lemma}), the term
under the square root in the second term of the generalization bound
can be reformulated in terms of the R\'enyi divergence of order 3 as:
$\sum_{k = 1}^c \frac{m_k r_{k, 2}^2}{\rho_k^2} = \frac{\ov
  r^2}{\rho^2} e^{2 \sfD_3\paren*{\sfr \, \| \, \frac{\brho}{\rho}
}}$, where $\sfr = \bracket[\bigg]{ \frac{m_k^{\frac{1}{3}} r_{k,
      2}^{\frac{2}{3}}} {\ov r^{\frac{2}{3}}}}_{k}$. 
      Thus, while
$\rho_k$s can be freely searched over a range of values in our general
algorithm (\ref{eq:alg-gen}), for experimental efficiency it may be beneficial to focus the search for the vector
$[\rho_k/\rho]_k$ near $\sfr$.
When the number of classes $c$ is very large, the search space can also be
significantly reduced by assigning identical $\rho_k$ values to
underrepresented classes while reserving distinct $\rho_k$ values for
the most frequently occurring classes.

\ignore{
When the number of classes $c$
is very large, the search space can be further reduced by constraining
the $\rho_k$ values for underrepresented classes to be identical and
allowing distinct $\rho_k$ values only for the most frequently
occurring classes.
}

\section{Formal Analysis of Some Core Methods}
\label{sec:existing-methods}

This section analyzes two popular methods presented in the
literature for tackling imbalanced data.

{\bf Resampling or cost-sensitive loss minimization.} A common approach for handling imbalanced data in practice is to
assign distinct costs to positive and negative samples. This
technique, implemented either explicitly or through resampling, is
widely used in empirical studies
\citep{chawla2002smote,he2009learning,he2013imbalanced,huang2016learning,
  buda2018systematic,cui2019class}. The associated target loss
$\sfL_{c_{+}, c_{-}}(h, x, y)$ can be expressed as follows, for any
$c_{+} > 0$, $c_{-} > 0$ and $(h, x, y) \in \sH_{\mathrm{all}} \times
\sX \times \sY$:
\begin{equation*}
c_{+} \ell_{0-1}(h, x, y) \1_{y  = +1} + c_{-} \ell_{0-1}(h, x, y) \1_{y  = -1}.
\end{equation*}
The following negative result, see also Appendix~\ref{app:negative}, shows that this loss function does not benefit
from a consistency, a motivating factor for our study of the class-imbalanced margin loss, Section~\ref{sec:binary}, with 
strong consistency guarantees.
\begin{restatable}[Negative results for resampling and cost-sensitive
    methods]{theorem}{Negative}
\label{thm:negative}
If $c_{+} \neq c_{-}$, then $\sfL_{c_{+}, c_{-}}$ is not
Bayes-consistent with respect to $\ell_{0-1}$.
\end{restatable}

{\bf Algorithms of \citep{cao2019learning}.} The theoretical analysis of \citet{cao2019learning} is limited to the
special case of binary classification with linear hypotheses in the
separable case. They propose an algorithm based on distinct positive
and negative \emph{geometric margins}, justified by their analysis. (Note that our  analysis is grounded in the more
general notion of \emph{confidence margins} and applies to both
separable and non-separable cases, and to general hypothesis
sets.)

\ignore{and yields inferior empirical performance on the standard zero-one
misclassification loss compared to our theory-based
algorithm, even in the separable binary classification case they
consider. }

\begin{figure}[t]
\vskip -0.05in
  \centering
  \includegraphics[scale=0.3]{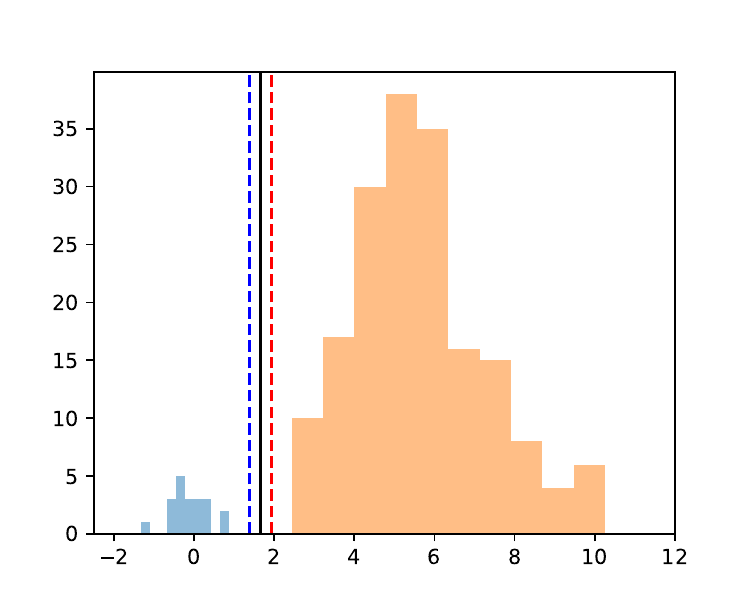}
  \hskip .02in
  \includegraphics[scale=0.3]{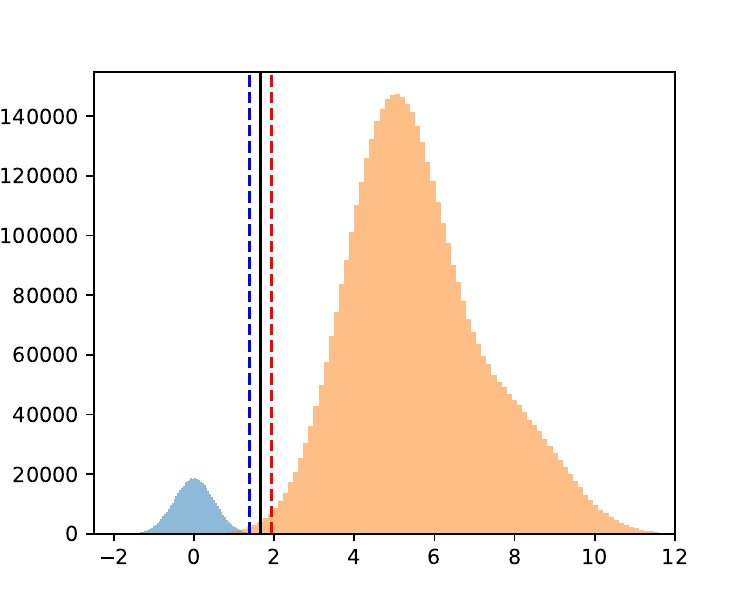}
 \vskip -0.20in
 \caption{Solutions in the separable case. Left:
    Empirical data with negative (blue) and positive (orange)
    points. The black line is the SVM solution, the red dashed line is
    \citet{cao2019learning}'s solution, and the blue dashed line is
    ours. Right: Full data distribution showing our solution achieves the 
    lowest generalization error.}
  \label{fig:sep}
\end{figure}
Their analysis contradicts the recommendations of our theory. Indeed, it is instructive to compare our margin values in the
separable case with those derived from the analysis of
\citet{cao2019learning}, in the special case they consider. The margin
values proposed in their work are:
\begin{align*}
  & \rho_{+} = \frac{2m^{\frac{1}{4}}_{-}}{m^{\frac{1}{4}}_{+} + m^{\frac{1}{4}}_{-}}
  \rho_{\rm{geom}},
  & \rho_{-} = \frac{2m^{\frac{1}{4}}_{+}}{m^{\frac{1}{4}}_{+} + m^{\frac{1}{4}}_{-}}
  \rho_{\rm{geom}}.
\end{align*}
Thus, disregarding the suboptimal exponent of $\frac{1}{4}$ compared
to $\frac{1}{3}$, which results from a less precise technical
analysis, the margin values recommended in their work directly
contradict those suggested by our analysis, see Eqn.~(\ref{eq:optimal-separable}). Specifically, their
analysis advocates for a smaller positive margin when $m_+ > m_-$,
whereas our theoretical analysis prescribes the opposite.  This
discrepancy stems from the analysis in \citep{cao2019learning}, which
focuses on the \emph{balanced loss}, which equalizes the impact of each class by weighting errors inversely to class frequency, and deviates fundamentally from the
standard zero-one loss we consider. 
Figure~\ref{fig:sep} illustrates these contrasting solutions in
a specific case of separable data. On the standard zero-one loss, our approach obtains a lower error.

Although their analysis is restricted to the linearly separable binary
case, the authors extend their work to the non-separable multi-class
setting by introducing a loss function (\LDAM) and algorithm\ignore{, without
any theoretical guarantees}.
Their loss function is an instance of the family of logistic loss
modifications, with an additive class label-dependent parameter
$\Delta_k = C/m_k^{1/4}$ inspired by their analysis in the
separable case, where $k$ denotes the label and $C$ a hyperparameter.
In the next section, we will compare our proposed algorithm with this technique as well as a number of other baselines.

\ignore{
The algorithm in \citep{cao2019learning} can be described
as follows:
\begin{equation*}
  \min_{w \in \Rset^d} \lambda \norm*{w}^2
  + \frac{1}{m} \sum_{k = 1}^c \sum_{i \in I_k}
  \log \bracket*{1 + \sum_{k' \neq k} e^{h(x_i, k') - h(x_i, k) + \Delta_k}},
\end{equation*}
where $\Delta_k = \frac{C}{m_k^{1/4}}$ for $k \in \curl*{1, \ldots,
  c}$, with $m_k$ being the number of training instances labeled as
$k$, and $C$ being a hyperparameter that can be selected via
cross-validation.
}

\begin{table}[t]
\caption{Accuracy of ResNet-34 on \emph{long-tailed} imbalanced CIFAR-10, CIFAR-100 and
  Tiny ImageNet; Means $\!\pm\!$ standard deviations over five runs for
  \IMMAX\ and a number of baseline techniques.}
   \ignore{
   \caption{Accuracy of ResNet-34 on long-tailed imbalanced CIFAR-10, CIFAR-100 and
  Tiny ImageNet; Mean $\pm$ standard deviation over five runs for the standard cross-entropy (\CE) loss, \FOCAL\ loss
  \citep{ross2017focal}, Re-Weighting (\RW), \LDAM\ loss \citep{cao2019learning}, Class-Balanced (\CB) loss \citep{cui2019class}, Balanced Softmax (\BS) loss \citep{jiawei2020balanced}, Equalization loss
\citep{tan2020equalization}, Logit Adjusted (\LA) loss \citep{menonlong} and \IMMAX\ loss;
  \IMMAX\ outperforms the baselines in all settings.}}
\centering
\resizebox{.98\columnwidth}{!}{
  \begin{tabular}{@{\hspace{0pt}}lllll@{\hspace{0pt}}}
  \toprule
    Method &  Ratio & CIFAR-10 & CIFAR-100 & Tiny ImageNet  \\
    \midrule
    \CE\ & \multirow{9}{*}{200} & 94.81 $\pm$ 0.38 & 78.78 $\pm$ 0.49 & 61.72 $\pm$ 0.68 \\
    \RW\ & & 92.36 $\pm$ 0.11 & 67.52 $\pm$ 0.76 & 48.16 $\pm$ 0.72  \\
    \BS\ & & 93.62 $\pm$ 0.25 & 72.27 $\pm$ 0.73 & 54.18 $\pm$ 0.65  \\
    \EQUAL\ & & 94.21 $\pm$ 0.21 & 76.23 $\pm$ 0.80 & 60.63 $\pm$ 0.85    \\
    \LA\ & & 94.59 $\pm$ 0.45 & 78.54 $\pm$ 0.49 & 61.83 $\pm$ 0.78 \\
    \CB\ & & 94.95 $\pm$ 0.46 & 79.36 $\pm$ 0.81 & 62.51 $\pm$ 0.71  \\
    \FOCAL\ & & 94.96 $\pm$ 0.39 & 79.53 $\pm$ 0.75 & 62.70 $\pm$ 0.79 \\
    \LDAM\ & & 95.45 $\pm$ 0.38 & 79.18 $\pm$ 0.71 & 63.70 $\pm$ 0.62 \\
    \textbf{\IMMAX} & & \textbf{96.11 $\pm$ 0.34} & \textbf{80.47 $\pm$ 0.68} & \textbf{65.20 $\pm$ 0.65} \\
    \cmidrule{1-1} \cmidrule{2-5}
    \CE\ & \multirow{9}{*}{100} & 95.65 $\pm$ 0.23 & 70.05 $\pm$ 0.36 & 51.17 $\pm$ 0.66\\
    \RW\ & & 93.32 $\pm$ 0.51 & 63.35 $\pm$ 0.26 & 43.73 $\pm$ 0.54 \\
    \BS\ & & 94.80 $\pm$ 0.26 & 65.36 $\pm$ 0.69 & 47.06 $\pm$ 0.73 \\
    \EQUAL\ & & 95.15 $\pm$ 0.39 & 68.81 $\pm$ 0.29  & 50.34 $\pm$ 0.78 \\
    \LA\ & & 95.75 $\pm$ 0.17 & 70.19 $\pm$ 0.78 & 51.27 $\pm$ 0.57 \\
    \CB\ & & 95.83 $\pm$ 0.11 & 69.85 $\pm$ 0.75 & 51.58 $\pm$ 0.65 \\
    \FOCAL\ & & 95.72 $\pm$ 0.11 & 70.33 $\pm$ 0.42 & 51.66 $\pm$ 0.78\\
    \LDAM\ & & 95.85 $\pm$ 0.10 & 70.43 $\pm$ 0.52 & 52.00 $\pm$ 0.53\\
    \textbf{\IMMAX} & & \textbf{96.56 $\pm$ 0.18} & \textbf{71.51 $\pm$ 0.34} & \textbf{53.47 $\pm$ 0.72}\\
    \cmidrule{1-1} \cmidrule{2-5}
    \CE\ & \multirow{9}{*}{10} & 93.05 $\pm$ 0.18 & 70.43 $\pm$ 0.27 & 53.22 $\pm$ 0.42\\
    \RW\ & & 91.45 $\pm$ 0.26 & 67.35 $\pm$ 0.51 & 48.46 $\pm$ 0.78 \\
    \BS\ & & 91.84 $\pm$ 0.30 & 66.52 $\pm$ 0.39 & 51.22 $\pm$ 0.53  \\
    \EQUAL\ & & 92.30 $\pm$ 0.18 & 68.64 $\pm$ 0.60 & 51.77 $\pm$ 0.30    \\
    \LA\ & & 92.84 $\pm$ 0.43 & 70.16 $\pm$ 0.58 & 53.75 $\pm$ 0.20  \\
    \CB\ & & 92.96 $\pm$ 0.27 & 70.31 $\pm$ 0.63 & 53.66 $\pm$ 0.58 \\
    \FOCAL\ & & 93.09 $\pm$ 0.33 & 70.70 $\pm$ 0.36 & 53.26 $\pm$ 0.50\\
    \LDAM\ & & 93.16 $\pm$ 0.25 & 70.94 $\pm$ 0.29 & 53.61 $\pm$ 0.20\\
    \textbf{\IMMAX} & & \textbf{93.68 $\pm$ 0.12} & \textbf{71.93 $\pm$ 0.36} & \textbf{54.89 $\pm$ 0.44}\\
    \bottomrule
    \end{tabular}
    }
\label{tab:comparison-long}
\end{table}

\begin{table}[t]
\caption{Accuracy of ResNet-34 on \emph{step-imbalanced} CIFAR-10, CIFAR-100 and
  Tiny ImageNet; Means $\!\pm\!$ standard deviations over five runs for
  \IMMAX\ and a number of baseline techniques.}
\centering
\resizebox{.98\columnwidth}{!}{
    \begin{tabular}{@{\hspace{0pt}}lllll@{\hspace{0pt}}}
    \toprule
      Method & Ratio & CIFAR-10 & CIFAR-100 & Tiny ImageNet \\
    \midrule
    \CE\ & \multirow{9}{*}{200} & 94.71 $\pm$ 0.24 & 77.07 $\pm$ 0.55 & 61.61 $\pm$ 0.53\\
    \RW\ & & 90.31 $\pm$ 0.38 & 72.59 $\pm$ 0.26 & 58.49 $\pm$ 0.61\\
    \BS\ & & 90.69 $\pm$ 0.41 & 74.18 $\pm$ 0.62 &  61.11 $\pm$ 0.32 \\
    \EQUAL\ & & 93.43 $\pm$ 0.23 & 76.85 $\pm$ 0.38 & 61.81 $\pm$ 0.39 \\
    \LA\ & & 94.85 $\pm$ 0.18 & 76.89 $\pm$ 0.74 &  61.51 $\pm$ 0.78\\
    \CB\ & & 94.92 $\pm$ 0.18 & 77.04 $\pm$ 0.13 & 61.55 $\pm$ 0.57 \\
    \FOCAL\ & & 94.78 $\pm$ 0.16 & 77.10 $\pm$ 0.62 & 61.77 $\pm$ 0.51\\
    \LDAM\ & & 94.85 $\pm$ 0.23 & 77.18 $\pm$ 0.50 & 62.54 $\pm$ 0.51\\
    \textbf{\IMMAX}  & & \textbf{95.42 $\pm$ 0.30} & \textbf{78.21 $\pm$ 0.48} & \textbf{63.57 $\pm$ 0.36} \\
    \cmidrule{1-1} \cmidrule{2-5}
    \CE\ & \multirow{9}{*}{100} & 95.03 $\pm$ 0.21 & 76.92 $\pm$ 0.27 & 60.62 $\pm$ 0.53 \\
    \RW\ & & 90.74 $\pm$ 0.19 & 68.17 $\pm$ 0.82 & 53.24 $\pm$ 0.65 \\
    \BS\ & & 93.24 $\pm$ 0.36 & 70.97 $\pm$ 0.35 & 60.07 $\pm$ 0.23 \\
    \EQUAL\ & & 94.04 $\pm$ 0.30 & 77.17 $\pm$ 0.20 & 60.46 $\pm$ 0.64 \\
    \LA\ & & 94.83 $\pm$ 0.11 & 77.27 $\pm$ 0.34 & 60.81 $\pm$ 0.46 \\
    \CB\ & & 95.08 $\pm$ 0.28 & 76.88 $\pm$ 0.44 & 60.63 $\pm$ 0.37 \\
    \FOCAL\ & & 95.07 $\pm$ 0.34 & 77.00 $\pm$ 0.34 & 60.72 $\pm$ 0.36\\
    \LDAM\ & & 95.17 $\pm$ 0.24 & 77.05 $\pm$ 0.45 & 62.33 $\pm$ 0.46\\
    \textbf{\IMMAX}  & & \textbf{96.05 $\pm$ 0.15} & \textbf{78.17 $\pm$ 0.35} & \textbf{63.04 $\pm$ 0.60}\\
    \cmidrule{1-1} \cmidrule{2-5}
    \CE\ & \multirow{9}{*}{10} & 92.95 $\pm$ 0.18 & 74.43 $\pm$ 0.38 & 59.68 $\pm$ 0.29 \\
    \RW\ & & 90.64 $\pm$ 0.15 & 68.65 $\pm$ 0.49 & 46.97 $\pm$ 0.73\\
    \BS\ & & 92.55 $\pm$ 0.26 & 69.55 $\pm$ 0.84 & 56.70 $\pm$ 0.34 \\
    \EQUAL\ & & 92.62 $\pm$ 0.24 & 72.64 $\pm$ 0.61 & 60.34 $\pm$ 0.52\\
    \LA\ & & 93.55 $\pm$ 0.30 & 74.60 $\pm$ 0.26 & 60.36 $\pm$ 0.28 \\
    \CB\ & & 93.54 $\pm$ 0.15 & 74.63 $\pm$ 0.36 & 59.88 $\pm$ 0.29\\
    \FOCAL\ & & 93.11 $\pm$ 0.16 & 74.51 $\pm$ 0.41 & 59.75 $\pm$ 0.44\\
    \LDAM\ & & 93.34 $\pm$ 0.16 & 74.82 $\pm$ 0.46 & 61.11 $\pm$ 0.30\\
    \textbf{\IMMAX}  & & \textbf{93.93 $\pm$  0.18} & \textbf{75.86 $\pm$ 0.26} & \textbf{61.93 $\pm$ 0.25}\\
    \bottomrule
    \end{tabular}
}
\label{tab:comparison-step}
\end{table}

\section{Experiments}
\label{sec:experiments}

In this section, we present experimental results for our
\IMMAX\ algorithm, comparing it to baseline methods in minimizing the standard
zero-one misclassification loss on CIFAR-10, CIFAR-100
\citep{Krizhevsky09learningmultiple} and Tiny ImageNet \citep{le2015tiny}
datasets. \ignore{Note that \IMMAX\ is not optimized
for other objectives, such as the balanced loss, and thus is not
expected to outperform state-of-the-art methods tailored to those
metrics.}.

Starting with multi-class classification, we strictly followed the experimental setup of \citet{cao2019learning}, adopting the same training procedure and neural network architectures. Specifically, we used ResNet-34 with ReLU activations \citep{he2016deep}, where ResNet-$n$ denotes a residual network with $n$ convolutional layers. For CIFAR-10 and CIFAR-100, we applied standard data augmentations, including 4-pixel padding followed by $32 \times 32$ random crops and random horizontal flips. For Tiny ImageNet, we used 8-pixel padding followed by $64 \times 64$ random crops.
All models were trained using Stochastic Gradient Descent (SGD) with Nesterov momentum \citep{nesterov1983method}, a batch size of $1,024$, and a weight decay of $1\times 10^{-3}$. Training spanned $200$ epochs, using a cosine decay learning rate schedule \citep{loshchilov2022sgdr} without restarts, with the initial learning rate set to $0.2$. For all the baselines and the \IMMAX\ algorithm, the hyperparameters were selected through cross-validation, see Appendix~\ref{app:experiments} for details.

To create imbalanced versions of the datasets, we reduced the percent of examples per class identically in the training and test sets. Following \citep{cao2019learning}, we consider two types of imbalances: long-tailed imbalance \citep{cui2019class} and step imbalance \citep{buda2018systematic}. The imbalance ratio, $ \frac{\max_{k = 1}^c m_k }{\min_{k = 1}^c m_k }$, represents the ratio of sample sizes between the most frequent and least frequent classes. In the long-tailed imbalance setting, class sample sizes decrease exponentially across classes. In the step setting, minority classes all have the same sample size, as do the frequent classes, creating a clear distinction between the two groups.

We compare our \IMMAX\ algorithm with widely used baselines, including the cross-entropy (\CE) loss, Re-Weighting (\RW) method \citep{xie1989logit,morik1999combining}, Balanced Softmax (\BS) loss \citep{jiawei2020balanced}, Equalization loss
\citep{tan2020equalization}, Logit Adjusted (\LA) loss \citep{menonlong}, Class-Balanced (\CB) loss \citep{cui2019class}, the \FOCAL\ loss in
  \citep{ross2017focal} and the \LDAM\ loss in \citep{cao2019learning} also detailed in Appendix~\ref{app:experiments}. We average
accuracies on the imbalanced test set  over five runs and report the means and standard deviations. \ignore{Experimental details on cross-validation are provided in Appendix~\ref{app:experiments}.} Note that \IMMAX\ is not optimized
for other objectives, such as the balanced loss, and thus is not
expected to outperform state-of-the-art methods tailored to those
metrics. 

Table~\ref{tab:comparison-long} and Table~\ref{tab:comparison-step} highlight that \IMMAX\ consistently outperforms all baseline methods on both the long-tailed and step-imbalanced datasets across all evaluated imbalance ratios (200, 100, and 10). 
In every scenario, \IMMAX\ achieves an absolute accuracy improvement of at least 0.6\% over the runner-up algorithm.
\ignore{To further demonstrate the value of our approach, we conducted experiments on additional datasets, including CIFAR-100 and Tiny ImageNet.
On CIFAR-100, \IMMAX\ outperforms \LDAM\ by at least 0.9\% in absolute accuracy for the long-tailed imbalanced version and by at least 1.0\% for the step-imbalanced version. Similarly, on Tiny ImageNet, \IMMAX\ surpassed \LDAM\ by 1.3\% on the long-tailed imbalanced version and by at least 0.7\% on the step-imbalanced version.}Note, that for the long-tailed distributions, the more imbalanced the dataset is, the more beneficial \IMMAX\  becomes compared to the baselines.

Finally, in Table~\ref{tab:comparison-binary}, we include binary classification results on CIFAR-10 obtained by classifying one category, e.g., airplane versus all the others using linear models. Table~\ref{tab:comparison-binary} shows that \IMMAX\ outperforms baselines.

Let us emphasize that our work is based on a novel, principled surrogate loss function designed for imbalanced data. Accordingly, we compare our new loss function directly against existing ones without incorporating additional techniques. However, all these loss functions, including ours, can be combined with existing data modification methods such as oversampling \citep{chawla2002smote} and undersampling \citep{WallaceSmallBrodleyTrikalinos2011,KubatMatwin1997}, as well as  optimization strategies like the deferred re-balancing schedule proposed in \citep{cao2019learning}, to further enhance performance. For a fair comparison of loss functions, we deliberately excluded these techniques from our experiments.

\section{Discussion}

\textbf{Balanced accuracy.} Our work focuses on the standard and unmodified zero-one misclassification loss, which remains the primary objective in many machine learning applications, as discussed in Section~\ref{sec:introduction}. Accordingly, we report standard accuracy based on this loss function in Section~\ref{sec:experiments}. In contrast, some previous studies (e.g., \citet{cao2019learning}) report ``balanced accuracy'', which equalizes the impact of each class by weighting errors inversely to class frequency.

\textbf{Selection of $\rho_k$.} 
As noted in Section~\ref{sec:multiclass}, while $\rho_k$ can be freely tuned, the search can be guided by the vector $[\rho_k/\rho]_k$ near $\sfr$, the optimal values in the separable case, a strategy we used in our experiments. For large multi-class settings, we further reduced the search space by assigning shared $\rho_k$ to rare classes and distinct ones to frequent classes, improving practicality with minimal performance impact.

\textbf{Balanced data.}
\IMMAX\ extends several classical algorithms, such as SVM \citep{CortesVapnik95} and logistic regression \citep{Verhulst1838}, 
to the more general setting of imbalanced data. When the training dataset is balanced, the theoretically optimal values of $\rho_k$ coincide across all classes, and \IMMAX\ reduces to standard methods. For example, with the logistic loss, it recovers the standard softmax cross-entropy loss with a suitable regularization parameter, thus yielding the same performance.

\textbf{Large-scale datasets.}  The dependency of our solution on sample size and dimensionality is similar to that of standard neural networks trained with cross-entropy loss (that is the logistic loss when softmax is applied to logits). Thus, our approach remains practical for large-scale datasets when using optimizers such as SGD \citep{bottou2010large}, AdaGrad \citep{duchi2011adaptive} or Adam \citep{kingma2014adam}. Our solution does depend on the number of classes, but this dependency is inherent to standard multi-class neural network solutions as well.

\textbf{Contrastive loss and temperature scaling}. The form of our \IMMAX\ loss function has some similarity with supervised contrastive losses (e.g., \citet{khosla2020supervised}), where a scalar temperature parameter is used in the inner product argument of the exponential. However, in our case, distinct parameters are introduced to allow different confidence margins across classes, serving a different purpose than in contrastive learning. Nevertheless, our margin analysis could provide a useful tool for analyzing contrastive learning. Alternatively, the \IMMAX\ loss can be viewed as a form of class-dependent temperature scaling derived from the logistic loss. In fact, our theoretical framework offers insight into the role of temperature parameters more broadly.

\section{Conclusion}
\label{sec:conclusion}

\begin{table}[t]
\vskip -0.075in
\caption{Accuracy of linear models on binarized version of CIFAR-10;  Means $\pm$ standard deviations for hinge loss, \IMMAX\ and \LDAM.}
\centering
\vskip 0.05in
\resizebox{.95\columnwidth}{!}{
    \begin{tabular}{@{\hspace{0pt}}llll@{\hspace{0pt}}}
    \toprule
     Method & Airplane & Automobile & Horse \\
    \midrule
    \hinge & 90.17 $\pm$ 0.09 & 91.01 $\pm$ 0.13 & 90.58 $\pm$ 0.11\\
    \LDAM\ & 90.37 $\pm$ 0.01 & 90.44 $\pm$ 0.02 & 90.17 $\pm$ 0.01\\
    \textbf{\IMMAX} & \textbf{91.02 $\pm$ 0.06} & \textbf{91.26 $\pm$ 0.05} & \textbf{91.03 $\pm$ 0.03}\\
    \bottomrule
    \end{tabular}
}
\vskip -0.28in
\label{tab:comparison-binary}
\end{table}

We introduced a rigorous theoretical framework for addressing class
imbalance, culminating in the class-imbalanced margin loss and
\IMMAX\ algorithms for binary and multi-class classification.
These algorithms are grounded in strong theoretical guarantees,
including $\sH$-consistency and robust generalization bounds.
Empirical results confirm that our algorithms outperform existing
methods while remaining aligned with key theoretical principles.
Our analysis is not limited to misclassification loss and can be
adapted to other objectives like balanced loss, offering broad
applicability. We believe these contributions offer a significant step
towards principled solutions for class imbalance across a diverse
range of machine learning applications.

\section*{Acknowledgements}

We thank the anonymous reviewers for their valuable feedback and constructive suggestions.

\section*{Impact Statement}

This paper presents work whose goal is to advance the field of 
Machine Learning. There are many potential societal consequences 
of our work, none which we feel must be specifically highlighted here.

\bibliography{lid,add}
\bibliographystyle{icml2025}

\newpage
\appendix
\onecolumn

\renewcommand{\contentsname}{Contents of Appendix}
\tableofcontents
\addtocontents{toc}{\protect\setcounter{tocdepth}{3}} 
\clearpage

\section{Related Work}
\label{app:related-work}

This section provides an expanded discussion of related work on class
imbalance in machine learning.

The class imbalance problem, defined by a significant disparity in the
number of instances across classes within a dataset, is a common
challenge in machine learning applications
\citep{Lewis:1994,Fawcett:1996,KubatMatwin1997,kang2021exploring,
  menonlong, liu2019large,cui2019class}.  This issue is prevalent
in many real-world binary classification scenarios, and arguably even
more so in multi-class problems with numerous classes. In such cases,
a few majority classes often dominate the dataset, leading to a
``long-tailed'' distribution.  Classifiers trained on these imbalanced
datasets often struggle, performing similarly to a naive baseline that
simply predicts the majority class.

The problem has been widely studied in the literature
\citep{CardieNowe1997,KubatMatwin1997,chawla2002smote,he2009learning,
  WallaceSmallBrodleyTrikalinos2011}. It includes numerous methods
including standard Softmax, class-sensitive learning, {Weighted
  Softmax}, weighted 0/1 loss
\citep{GabidollaZharmagambetovCarreiraPerpinan2024}, size-invariant
metrics for Imbalanced Multi-object Salient Object Detection studied
by \citet{LiXuBaoYangCongCaoHuang2024} as well as {Focal loss}
\citep{lin2017focal}, {LDAM} \citep{cao2019learning}, {ESQL}
\citep{tan2020equalization}, {Balanced Softmax}
\citep{jiawei2020balanced}, {LADE} \citep{hong2020disentangling}),
logit adjustment ({UNO-IC} \citep{tian2020posterior}, {LSC} \citep{weilearning}), transfer
learning ({SSP} \citep{yang2020rethinking}), data augmentation ({RSG}
\citep{wang2021rsg}, {BSGAL} \citep{zhugenerative}, {ELTA} \citep{liuelta}, {OT} \citep{gao2024enhancing}), representation learning ({OLTR}
\citep{liu2019large}, {PaCo} \citep{cui2021parametric}, {DisA} \citep{gao2024distribution}, {RichSem} \citep{meng2024learning}, {RBL} \citep{meng2024learning}, {WCDAS} \citep{han2023wrapped}),  classifier
design ({De-confound} \citep{tang2020long}, \citep{yang2022inducing,kasarla2022maximum},  {LIFT} \citep{shi2024long}, {SimPro} \citep{dusimpro}), decoupled training
({Decouple-IB-CRT} \citep{kang2019decoupling}, {CB-CRT}
\citep{kang2019decoupling}, {SR-CRT} \citep{kang2019decoupling},
      {PB-CRT} \citep{kang2019decoupling}, {MiSLAS}
      \citep{zhong2021improving}), ensemble learning ({BBN}
      \citep{zhou2020bbn}, {LFME} \citep{xiang2020learning}, {RIDE}
      \citep{wang2020long}, {ResLT} \citep{cui2021reslt}, {SADE}
      \citep{zhang2022self}, {DirMixE} \citep{yangharnessing}).  An interesting recent study
      characterizes the asymptotic performances of linear classifiers
      trained on imbalanced datasets for different metrics
      \citep{LoffredoPastoreCoccoMonasson2024}.
      
Due to space restrictions, we
cannot give a detailed discussion of all these methods. Instead, we
will describe and discuss several broad categories of existing methods
to tackle this problem and refer to reader to a recent survey of
\citet{ZhangKangHooiYanFeng2023} for more details.
These methods fall into the following broad categories.

\textbf{Data modification methods.} These include methods such as
oversampling the minority class \citep{chawla2002smote}, undersampling
the majority class
\citep{WallaceSmallBrodleyTrikalinos2011,KubatMatwin1997}, or
generating synthetic samples (e.g., SMOTE
\citep{chawla2002smote,QiaoLiu2008,han2005borderline}), aim to
rebalance the dataset before training
\citep{chawla2002smote,estabrooks2004multiple,
  liu2008exploratory,zhang2021learning,shi2023re}.

\textbf{Cost-sensitive techniques.} These techniques, including
cost-sensitive learning and the incorporation of class weights assign
different penalization costs to losses on different classes. They
include cost-sensitive SVM \citep{Iranmehr:2019,Masnadi-Shirazi:2010}
and other cost-senstive methods
\citep{elkan2001foundations,zhou2005training, zhao2018adaptive,
  zhang2018online, zhang2019online,sun2007cost,Fan:2017,
  jamal2020rethinking,zhang2022self,wang2022solar,xiao2024fed,suh2023long}. The weights are often determined by the
relative number of samples in each class or a notion of effective
sample size \cite{cui2019class}.

These two method categories are very related and can actually be shown
to be equivalent in the limit.  Cost-sensitive methods can be viewed
as more efficient, flexible and principled techniques for implementing
data sampling methods.  However, these methods often risk overfitting
the minority class or discarding valuable information from the
majority class.  Both methods inherently bias the input training data
distribution and suffer from Bayes-inconsistency (in Section~\ref{sec:existing-methods}, we prove
that cost-sensitive methods do not admit Bayes-consistency). While
they have been both reported to be effective in various instances,
this varies and depends on the problem, the distribution, the choice
of predictors, and the performance metric adopted and they have been
reported not to be effective in all cases \citep{VanHulse:2007}.
Additionally, cost-sensitive methods often resort to careful tuning of
hyperparameters.  Hybrid approaches attempt to combine the strengths
of data modification and cost-sensitive methods but often inherit
their respective limitations.

\textbf{Logistic loss modifications.}  A family of more recent methods
rely on logistic loss modifications.  They consist of modifying the
logistic loss by augmenting each logit (or predicted score) with an
additive hyperparameter.  They can be equivalently described as a
cost-sensitive modification of the exponential terms appearing in the
definition of the logistic loss.  They include the Balanced Softmax
loss \cite{jiawei2020balanced}, the Equalization loss
\cite{tan2020equalization}, and the \LDAM\ loss
\cite{cao2019learning}. Other similar additive change methods use
quadratically many hyperparameters with a distinct additive parameter
for each pair of logits. They include the logit adjustment methods of
\citet{menonlong} and \citet{khan2019striking}.
\citet{menonlong} argue that their specific choice of the
hyperparameter values is Bayes-consistent.
A multiplicative modification of the logits, with one hyperparameter
per class label is advocated by \citet{Ye:2020}. This can be
equivalently viewed as normalizing scoring functions (or feature
vectors in the linear case) beforehand, which is a standard method
used in many learning applications, irrespective of the presence of
imbalanced classes. The Vector-Scaling loss of \citet{kini2021label}
combines the additive modification of the logits with this
multiplicative change. These authors further present an analysis of
this method in the case of linear predictors, underscoring the
specific benefits of the multiplicative changes.  As already pointed
out, the multiplicative changes coincide with prior rescaling or
renormalization of the feature vectors, however.

\textbf{Other methods.} Additional approaches for tackling imbalanced
datasets (see \citet{ZhangKangHooiYanFeng2023}) include post-hoc
correction of decision thresholds \citep{Fawcett:1996,Collell:2016} or
weights \citep{kang2019decoupling,Kim:2019}], as well as information and data
  augmentation via transfer learning, or
  distillation \citep{LiLiYeZhang2024}.

Despite significant advances, these techniques face persistent
challenges.

First, most existing solutions are heuristic-driven and lack a solid
theoretical foundation, making their performance difficult to predict
across varying contexts. 

In fact, we are not aware of any analysis of
the generalization guarantees for these methods, with the exception of
that of \cite{cao2019learning,jiawei2020balanced,wang2023unified}. However, as further discussed in
Section~\ref{sec:existing-methods}, the analysis presented by these
authors is limited to the \emph{balanced loss} \citep{cortes2025improved}, which
equalizes the impact of each class by weighting errors inversely to class 
frequency. More specifically,
the analysis in \citep{cao2019learning} is limited to binary classification 
and holds only for the
separable case.
The balanced loss function differs from the target misclassification
loss. It has been argued, and that is important, that the balanced
loss admits beneficial fairness properties when class labels correlate
with demographic attributes as it treats all class errors equally.
The balanced loss is also the metric considered in the analysis of
several of the logistic loss modifications papers
\citep{cao2019learning,menonlong,Ye:2020,kini2021label}.
However, class labels do not alway relate to demographic attributes.
Furthermore, many other criteria are considered for fairness purposes
and in many machine learning applications, the misclassification
remains the key target loss function to minimize.  We will show that,
even in the special case of the analysis of \cite{cao2019learning},
the solution they propose is the opposite of the one corresponding to
our theoretical analysis for the standard misclassification loss. We
further show that their solution is empirically outperformed by ours.

Second, the evaluation of these methods is frequently biased toward
alternative metrics such as F1-measure, AUC, or other metrics
weighting false or true positive rate differently, which may obscure
their true effectiveness on standard misclassification.  Additionally,
these methods often seem to struggle with extreme imbalances or when
the minority class exhibits high intra-class variability.

We refer to \citet{ZhangKangHooiYanFeng2023} for more details about
work related to learning from imbalanced data. 

\section{Experimental details}
\label{app:experiments}

In this section, we provide further experimental details. We first discuss the loss functions for the baselines and then provide ranges of hyperparameters tested via cross-validation. 

\textbf{Baseline algorithms.}
In Section~\ref{sec:experiments}, we compared our \IMMAX\ algorithm with well-known baselines, including the cross-entropy (\CE) loss, Re-Weighting (\RW) method \citep{xie1989logit,morik1999combining}, Balanced Softmax (\BS) loss \citep{jiawei2020balanced}, Equalization loss
\citep{tan2020equalization}, Logit Adjusted (\LA) loss \citep{menonlong}, Class-Balanced (\CB) loss \citep{cui2019class}, the \FOCAL\ loss in
  \citep{ross2017focal} and the \LDAM\ loss in \citep{cao2019learning}.

The \IMMAX\ algorithm optimizes the loss function:
\begin{equation*}
    \forall (h, x, y), \quad \sfL_{\IMMAX}(h, x, y) = \log\paren*{\sum_{j = 1}^c e^{\frac{h(x, j) - h(x, y)}{\rho_y}}},
    \end{equation*}
    where $\rho_k  > 0$ for $k \in [c]$ are hyperparameters. In comparison, the baselines optimize the following loss functions: \begin{itemize}
    \item Cross-entropy (\CE) loss:
    \begin{equation*}
    \forall (h, x, y), \quad \sfL_{\CE}(h, x, y) = -\log\paren*{\frac{e^{h(x, y)}}{\sum_{j = 1}^c e^{h(x, j)}}}.
    \end{equation*}
    \item Re-Weighting (\RW) method \citep{xie1989logit,morik1999combining}: Each sample is re-weighted by the inverse of its class's sample size and subsequently normalized such that the average weight within each mini-batch is 1. This is equivalent to minimizing the loss function given below:
    \begin{equation*}
    \forall (h, x, y), \quad \sfL_{\RW}(h, x, y) = - \frac{m}{m_{y}}\log\paren*{\frac{e^{h(x, y)}}{\sum_{j = 1}^c e^{h(x, j)}}}.
    \end{equation*}
    \item Balanced Softmax (\BS) loss \citep{jiawei2020balanced}:
    \begin{equation*}
    \forall (h, x, y), \quad \sfL_{\BS}(h, x, y) = -\log\paren*{\frac{m_{y} e^{h(x, y)}}{\sum_{j = 1}^c m_{j} e^{h(x, j)}}}.
    \end{equation*}
    \item Equalization loss \citep{tan2020equalization}:
    \begin{equation*}
    \forall (h, x, y), \quad \sfL_{\EQUAL}(h, x, y) = -\log\paren*{\frac{e^{h(x, y)}}{\sum_{j = 1}^c w_{j} e^{h(x, j)}}},
    \end{equation*}
    with the weight $w_j$ computed by $w_j = 1 - \beta \1_{\frac{m_j}{m} < \lambda} \1_{y \neq j}$, where $\beta \sim \text{Bernoulli}(p)$ is a Bernoulli distribution. Here, $1 > p > 0$ and $ 1 > \lambda > 0$ are two hyperparameters. 
    \item Logit Adjusted (\LA) loss \citep{menonlong}:
    \begin{equation*}
    \forall (h, x, y), \quad \sfL_{\LA}(h, x, y) = -\log\paren*{\frac{e^{h(x, y) + \tau \log(m_y)} }{\sum_{j = 1}^c e^{h(x, j) + \tau \log(m_j)}}},
    \end{equation*}
    where $\tau > 0$ is a hyperparameter.
    \item Class-Balanced (\CB) loss \citep{cui2019class}:
    \begin{equation*}
    \forall (h, x, y), \quad \sfL_{\CB}(h, x, y) = -\frac{1 - \gamma}{1 - \gamma^{\frac{m_y}{m}}}\log\paren*{\frac{e^{h(x, y)} }{\sum_{j = 1}^c e^{h(x, j)}}},
    \end{equation*}
    where $1 > \gamma > 0$ is a hyperparameter.
    \item \FOCAL\ loss in
  \citep{ross2017focal}:
    \begin{equation*}
    \forall (h, x, y), \quad \sfL_{\FOCAL}(h, x, y) = -\paren*{1 - \frac{e^{h(x, y)} }{\sum_{j = 1}^c e^{h(x, j)}}}^{\gamma}\log\paren*{\frac{e^{h(x, y)} }{\sum_{j = 1}^c e^{h(x, j)}}},
    \end{equation*}
    where $\gamma \geq 0$ is a hyperparameter.
    \item \LDAM\ loss in \citep{cao2019learning}: \begin{equation*}
    \forall (h, x, y), \quad \sfL_{\LDAM}(h, x, y) = -\log\paren*{\frac{e^{h(x, y) - \Delta_y}}{e^{h(x, y) - \Delta_y} + \sum_{j \neq y} e^{h(x, j)}}},
    \end{equation*}
    where $\Delta_j = \frac{C}{m_j^{\frac14}}$ for $j \in [c]$ and $C > 0$ is a hyperparameter.
\end{itemize}
\textbf{Discussion.} Among these baselines, \RW\ method, \CB\ loss, and \FOCAL\ loss are cost-sensitive methods, while \BS\ loss, \EQUAL\ loss, \LA\ loss, and \LDAM\ loss are logistic loss modification methods.  Note that when $\tau = 1$, the \LA\ loss is the same as the \BS\ loss; when $\tau = 0$, the \FOCAL\ loss is the same as the \CE\ loss. Also note that in the balanced setting where $m_j = {m} / c$ for $j \in [c]$, the \RW\ method, \BS\ loss, \LA\ loss and \CB\ loss are the same as the \CE\ loss. 

\textbf{Hyperparameter search.} As mentioned in Section~\ref{sec:experiments}, all hyperparameters were selected through cross-validation for all the baselines and the \IMMAX\ algorithm. More specifically, the parameter ranges for each method are as follows. Note that the \CE\ loss, \RW\ method and \BS\ loss do not have any hyperparameters.
\begin{itemize}
    \item \EQUAL\ loss: following \citep{tan2020equalization}, $p$ is chosen from $\curl*{0.1, 0.2, 0.3, 0.4, 0.5, 0.6, 0.7, 0.8, 0.9}$ and $\lambda$ is chosen from $\curl*{0.176, 0.5, 0.8, 1.5, 1.76, 2.0, 3.0, 5.0} \times 10^{-3}$.
    \item \LA\ loss: following \citep{menonlong}, $\tau$ is chosen from $\curl*{0.1, 0.2, 0.3, 0.4, 0.5, 0.6, 0.7, 0.8, 0.9, 1.0}$ and $\curl*{1.5, 2.0, 2.5, 3.0, 3.5, 4.0, 4.5, 5.0, 5.5, 6.0, 6.5, 7.0, 7.5, 8.0, 8.5, 9.0, 9.5, 10.0}$. When $\tau = 1$ (the suggested value in \citep{menonlong}),  the \LA\ loss is equivalent to the \BS\ loss. We observed improved performance for small values of $\tau < 1$ when minimizing the standard zero-one misclassification loss. Therefore, we conducted a finer search between $0$ and $1$.
    \item \CB\ loss: following \citep{cui2019class}, $\gamma$ is chosen from $\curl*{0.1, 0.2, 0.3, 0.4, 0.5, 0.6, 0.7, 0.8, 0.9, 0.99, 0.999, 0.9999}$. While the default values of  $\curl*{0.9, 0.99, 0.999, 0.9999}$  are suggested in \citep{cui2019class}, we observed that they are not effective for minimizing the standard zero-one misclassification loss.  We found that performance is typically better when $\gamma$ is close to $0$.
    \item \FOCAL\ loss: $\gamma$ is chosen from $\curl*{1.0, 1.5, 2.0, 2.5, 3.0, 3.5, 4.0, 4.5, 5.0, 5.5, 6.0, 6.5, 7.0, 7.5, 8.0, 8.5, 9.0, 9.5, 10.0}$ and $\curl*{0.0, 0.1, 0.2, 0.3, 0.4, 0.5, 0.6, 0.7, 0.8, 0.9}$ following \citep{ross2017focal}.  We observe that performance is typically better when $\gamma$ is less than $1$. Therefore, we conducted a finer search between $0$ and $1$.
    \item \LDAM\ loss: following \citep{cao2019learning}, $C$ is chosen from $\curl*{10^{-4}, 10^{-3}, 10^{-2}, 10^{-1}, 1.0, 10.0, 100.0, 1000.0, 10000.0}$ and $\curl*{5\times 10^{-4}, 5\times 10^{-3}, 5\times 10^{-2}, 5\times 10^{-1}, 5.0, 50.0, 500.0, 5000.0}$.
    \item \IMMAX\ loss: following Section~\ref{sec:algorithms} and Appendix~\ref{app:lemma}, each $\rho_k$ is searched within a symmetric interval centered at the theoretically optimal value $\rho_k^* = \frac{m_k^{\frac13}}{\sum_{j \in [c]} m_j^{\frac13}}$, with width $\delta_k$; that is, over the interval $\bracket*{\rho_k^* - \delta_k, \rho_k^* + \delta_k}$, where $\delta_k$ is chosen relative to the scale of $\rho_k^*$. In particular, we set $\delta_k \approx 0.8 \rho_k^*$ and evaluate approximately 10 values within each interval. Empirically, performance is not sensitive to variations within this neighborhood. In the step imbalanced setting, we assign identical $\rho_k$ values to
minority classes and distinct $\rho_k$ values to frequent classes before the search.
\end{itemize}

\section{Proof of Theorem~\ref{thm:negative}}
\label{app:negative}

\Negative*
\begin{proof}
Consider a singleton distribution concentrated at a point $x$. Without
loss of generality, assume that $c_{+} > c_{-} > 0$. Next, consider
the conditional distribution $\eta(x) = \Pr \bracket*{Y = +1 \mid X =
  x}$ denote the conditional probability that $Y = +1$ given $X = x$
with $\eta(x) = \frac12 - \e$, for $\e \in (0, \frac12)$. By the proof
of Theorem~\ref{thm:H-consistency-binary}, the best-in-class error for
the zero-one loss can be expressed as follows:
\begin{equation*}
\inf_{h \in \sH} \sR_{\ell_{0-1}}(h) = \eta(x),
\end{equation*}
which can be achieved by any $h^*_{\ell_{0-1}}$ such that
$h^*_{\ell_{0-1}}(x) < 0$, that is a hypothesis \emph{all-negative} on
$x$. For the cost-sensitive loss function $\sfL_{c_{+}, c_{-}}$, the
generalization error can be expressed as follows:
\begin{align*}
\sR_{\sfL_{c_{+}, c_{-}}}(h) = \eta(x) c_{+} \1_{h(x) < 0} + (1 - \eta(x)) c_{-} \1_{h(x) \geq 0}.
\end{align*}
Thus, for any $c_{+} > c_{-} > 0$, there exists $\e \in (0, \frac12)$
such that the following holds:
\begin{align*}
(1 - \eta(x)) c_{-} < \eta(x) c_{+} 
& \iff \frac{\frac12 + \e}{\frac12 - \e} < \frac{c_{+}}{c_{-}}\\
& \iff 0 < \e < \frac{\frac12 c_{+} - \frac12 c_{-}}{c_{+} + c_{-}} < \frac12,
\end{align*}
where we used the fact that $x \mapsto (1 - x) / x = 1/x - 1$ is a
bijection from $(0, 1]$ to $[0, +\infty)$.  For this $\e$, the
    best-in-class error of $\sfL_{c_{+}, c_{-}}$ is
\begin{equation*}
\inf_{h \in \sH} \sR_{\sfL_{c_{+}, c_{-}}}(h) = \paren*{1 - \eta(x)} c_{-},
\end{equation*}
which can be achieved by any \emph{all-positive} $h^*_{\sfL_{c_{+},
    c_{-}}}$ such that $h^*_{\sfL_{c_{+}, c_{-}}}(x) \geq 0$. Thus,
$h^*_{\sfL_{c_{+}, c_{-}}}$ differs from $h^*_{\ell_{0-1}}$, which
implies that $\sfL_{c_{+}, c_{-}}$ is not Bayes-consistent with
respect to $\ell_{0-1}$.
\end{proof}

\section{Binary Classification: Proofs}
\label{app:binary}

\subsection{Proof of Lemma~\ref{lemma:margin-loss}}
\label{app:margin-loss}

\MarginLoss*
\begin{proof}
When $y h(x) \leq 0$, we have $\Phi_{\rho_{+}}(y h(x)) = \Phi_{\rho_{-}}(y h(x)) = 1$, so the equality holds. When $y h(x) > 0$, we have $y > 0 \iff h(x) > 0$ and $y < 0 \iff h(x) < 0$, which also implies the equality.
\end{proof}

\subsection{Proof of Theorem~\ref{thm:H-consistency-binary}}
\label{app:H-consistency-binary}

\HConsistencyBinary*
\begin{proof}
Let $\eta(x) = \Pr \bracket*{Y = +1 \mid X = x}$ denote the conditional probability that $Y = +1$ given $X = x$. Without loss of generality, assume $\eta(x) \in [0, \frac12]$.
Then, the conditional error and the best-in-class conditional error of the zero-one loss can be expressed as follows:
\begin{align*}
\E_{y} \bracket*{\ell_{0-1}(h, x, y) \mid x} &=  \eta(x) \1_{h(x) < 0} + \paren*{1 - \eta(x)} \1_{h(x) \geq 0}\\
\inf_{h \in \sH }\E_{y} \bracket*{\ell_{0-1}(h, x, y) \mid x} &= \min \curl*{\eta(x), 1 - \eta(x)} = \eta(x).
\end{align*}
Furthermore, the difference between the two terms is given by:
\begin{align*}
\E_{y} \bracket*{\ell_{0-1}(h, x, y) \mid x} - \inf_{h \in \sH }\E_{y} \bracket*{\ell_{0-1}(h, x, y) \mid x} = \begin{cases}
1 - 2\eta(x) & h(x) \geq 0\\
0 & h(x) < 0
\end{cases}
\end{align*}
For the class-imbalanced margin loss, the conditional error can be expressed as follows:
\begin{align*}
\E_{y} \bracket*{\sfL_{\rho_{+}, \rho_{-}}(h, x, y) \mid x}
&=  \eta(x) \Phi_{\rho_{+}}(h(x))  + \paren*{1 - \eta(x)} \Phi_{\rho_{-}}(-h(x)) \\
& = \eta(x) \min \paren*{1, \max \paren*{0, 1 - \frac{h(x)}{\rho_{+}}}} + \paren*{1 - \eta(x)} \min \paren*{1, \max \paren*{0, 1 + \frac{h(x)}{\rho_{-}}}}\\
& = \begin{cases}
1 - \eta(x) & h(x) \geq \rho_{+}\\
\eta(x)\paren*{1 - \frac{h(x)}{\rho_{+}}} + \paren*{1 - \eta(x)} & \rho_{+} > h(x) \geq 0\\
\eta(x) + \paren*{1 - \eta(x)} \paren*{1 + \frac{h(x)}{\rho_{-}}} &  -\rho_{-} \leq h(x) < 0\\
 \eta(x) & h(x) < -\rho_{-}.
\end{cases}
\end{align*}
Thus, the best-in-class conditional error can be expressed as follows:
\begin{equation*}
\inf_{h \in \sH } \E_{y} \bracket*{\sfL_{\rho_{+}, \rho_{-}}(h, x, y) \mid x} = \min\curl*{\eta(x), 1 - \eta(x)} = \eta(x)
\end{equation*}
Consider the case where $h(x) \geq 0$. The difference between the two terms is given by:
\begin{align*}
\E_{y} \bracket*{\sfL_{\rho_{+}, \rho_{-}}(h, x, y) \mid x} - \inf_{h \in \sH }\E_{y} \bracket*{\sfL_{\rho_{+}, \rho_{-}}(h, x, y) \mid x} &= \begin{cases}
1 - 2\eta(x) & h(x) \geq \rho_{+}\\
\eta(x)\paren*{1 - \frac{h(x)}{\rho_{+}}} + 1 - 2\eta(x) & \rho_{+} > h(x) \geq 0
\end{cases}\\
& \geq 1 - 2\eta(x)\\
& = \E_{y} \bracket*{\ell_{0-1}(h, x, y) \mid x} - \inf_{h \in \sH }\E_{y} \bracket*{\ell_{0-1}(h, x, y) \mid x}.
\end{align*}
By taking the expectation of both sides, we obtain:
\begin{equation*}
\sR_{\ell_{0-1}}(h) - \sR^*_{\ell_{0-1}}(\sH) + \sM_{\ell_{0-1}}(\sH) \leq \sR_{\sfL_{\rho_{+}, \rho_{-}}}(h) - \sR^*_{\sfL_{\rho_{+}, \rho_{-}}}(\sH) + \sM_{\sfL_{\rho_{+}, \rho_{-}}}(\sH),
\end{equation*}
which completes the proof.
\end{proof}

\subsection{Proof of Theorem~\ref{thm:margin-bound-binary}}
\label{app:margin-bound-binary}

\MarginBoundBinary*
\begin{proof}
Consider the family of functions taking values in $[0, 1]$:
\begin{equation*}
\wt \sH = \curl*{z = (x, y) \mapsto \sfL_{\rho_{+}, \rho_{-}}(h, x, y) \colon h \in \sH}.
\end{equation*}
By \citep[Theorem~3.3]{MohriRostamizadehTalwalkar2018}, with probability at least $1 - \delta$, for all $g \in \wt \sH$,
\begin{equation*}
\E[g(z)] \leq \frac{1}{m} \sum_{i = 1}^m g(z_i) + 2 \Rad_{m}(\wt \sH) + \sqrt{\frac{\log \frac{1}{\delta}}{2m}},
\end{equation*}
and thus, for all $h \in \sH$,
\begin{equation*}
\E[\sfL_{\rho_{+}, \rho_{-}}(h, x, y)] \leq \h \sR_{S}^{\rho_{+}, \rho_{-}}(h) + 2 \Rad_{m}(\wt \sH) + \sqrt{\frac{\log \frac{1}{\delta}}{2m}}.
\end{equation*}
Since $\sR_{\ell_{0-1}}(h) \leq \sR_{\sfL_{\rho_{+}, \rho_{-}}}(h) = \E[\sfL_{\rho_{+}, \rho_{-}}(h, x, y)]$, we have
\begin{equation*}
\sR_{\ell_{0-1}}(h) \leq \h \sR_{S}^{\rho_{+}, \rho_{-}}(h) + 2 \Rad_{m}(\wt \sH) + \sqrt{\frac{\log \frac{1}{\delta}}{2m}}.
\end{equation*}
Since $\Phi_{\rho}$ is $\frac{1}{\rho}$-Lipschitz, by \citep[Lemma~5.7]{MohriRostamizadehTalwalkar2018}, $\Rad_{m}(\wt \sH)$ can be rewritten as follows:
\begin{align*}
\Rad_{m}(\wt \sH) &= \frac{1}{m} \E_{S, \sigma}\bracket*{\sup_{h \in \sH} \sum_{i = 1}^m \sigma_i \sfL_{\rho_{+}, \rho_{-}}(h, x_i, y_i)}\\
&= \frac{1}{m} \E_{S, \sigma}\bracket*{\sup_{h \in \sH} \sum_{i = 1}^m \sigma_i \bracket*{\Phi_{\rho_{+}}(y_i h(x_i)) \1_{y_i  = +1} + \Phi_{\rho_{-}}(y_i h(x_i)) \1_{y_i  = -1}}}\\
&\leq \frac{1}{m} \E_{S, \sigma}\bracket*{\sup_{h \in \sH} \curl*{ \frac{1}{\rho_{+}} \paren*{\sum_{i \in I_+} \sigma_i h(x_i)} + \frac{1}{\rho_{-}} \paren*{\sum_{i \in I_{-}} -\sigma_i h(x_i)}}}\\
&= \Rad_{m}^{\rho_{+}, \rho_{-}}(\sH),
\end{align*}
where the last equality stems from the fact that the variables $\sigma_i$ and $-\sigma_i$ are distributed in the same way.
This proves the first inequality. The second inequality, can be derived in the same way by using the second inequality of \citep[Theorem~3.3]{MohriRostamizadehTalwalkar2018}.
\end{proof}

\subsection{Uniform Margin Bound for Imbalanced Binary Classification}
\label{app:margin-bound-binary-uniform}

\begin{theorem}[Uniform margin bound for imbalanced binary classification]
\label{thm:margin-bound-binary-uniform}
Let $\sH$ be a set of real-valued functions. Fix $r_{+} > 0$ and $r_{-} > 0$. Then, for any $\delta > 0$, with probability at least $1 - \delta$, each of the following holds for all $h \in \sH$, $\rho_{+} \in (0, r_{+}]$ and $\rho_{-} \in (0, r_{-}]$:
\begin{align*}
\sR_{\ell_{0-1}}(h) &\leq \h \sR_{S}^{\rho_{+}, \rho_{-}}(h) + 4 \Rad_{m}^{\rho_{+}, \rho_{-}}(\sH) + \sqrt{\frac{\log \log_2 \frac{2r_{+}}{\rho_{+}}}{m}} + \sqrt{\frac{\log \log_2 \frac{2r_{-}}{\rho_{-}}}{m}} + \sqrt{\frac{\log \frac{4}{\delta}}{2m}}\\
\sR_{\ell_{0-1}}(h) &\leq \h \sR_{S}^{\rho_{+}, \rho_{-}}(h) + 4 \h \Rad_{S}^{\rho_{+}, \rho_{-}}(\sH) + \sqrt{\frac{\log \log_2 \frac{2r_{+}}{\rho_{+}}}{m}} + \sqrt{\frac{\log \log_2 \frac{2r_{-}}{\rho_{-}}}{m}} + 3 \sqrt{\frac{\log \frac{8}{\delta}}{2m}}.
\end{align*}
\end{theorem}
\begin{proof}
First, consider two sequences $\paren*{\rho_{+}^k}_{k \geq 1}$ and $\paren*{\e_k}_{k \geq 1}$, with $\e_k \in (0, 1]$. By Theorem~\ref{thm:margin-bound-binary}, for any fixed $k \geq 1$ and $\rho_{-} > 0$,
\begin{align*}
\Pr \bracket*{\sup_{h \in \sH} \sR_{\ell_{0-1}}(h) - \h \sR_{S}^{\rho_{+}^k, \rho_{-}}(h) > 2 \Rad_{m}^{\rho_{+}^k, \rho_{-}}(\sH) + \e_k} \leq e^{-2m \e_k^2}.
\end{align*}
Choosing $\e_k = \e + \sqrt{\frac{\log k}{m}}$, then, by the union bound, the following holds for any fixed $\rho_{-} > 0$:
\begin{align*}
& \Pr \bracket*{\sup_{\substack{h \in \sH\\ k\geq 1}} \sR_{\ell_{0-1}}(h) - \h \sR_{S}^{\rho_{+}^k, \rho_{-}}(h) - 2 \Rad_{m}^{\rho_{+}^k, \rho_{-}}(\sH) - \e_k > 0} \\
& \leq \sum_{k \geq 1}  e^{-2m \e_k^2} = \sum_{k \geq 1} \exp^{-2m \paren*{\e + \sqrt{\frac{\log k}{m}}}^2} \leq \sum_{k \geq 1} e^{-2m \e^2} e^{-2\log k} = \paren*{\sum_{k \geq 1} 1 / k^2} e^{-2m \e^2} \ignore{= \frac{\pi^2}{6} e^{-2m \e^2}} \leq 2 e^{-2m \e^2}.
\end{align*}
We can choose $\rho_{+}^k = r_{+} / 2^k$. For any $\rho_{+} \in (0, r_{+} ]$, there exists $k \geq 1$ such that $\rho_{+} \in (\rho_{+}^k, \rho_{+}^{k - 1}]$, with $\rho_{+}^0 = r_{+}$. For that $k$, $\rho_{+} \leq \rho_{+}^{k - 1} = 2 \rho_{+}^k$, thus $1 / \rho_{+}^k \leq 2 / \rho_{+}$ and $\sqrt{\log k} = \sqrt{\log \log_2 (r_{+} / \rho_{+}^k)} \leq \sqrt{\log \log_2 (2 r_{+} / \rho_{+})}$. Furthermore, for any $h \in \sH$ and $\rho_{-} > 0$, $\h \sR_{S}^{\rho_{+}^k, \rho_{-}}(h) \leq \h \sR_{S}^{\rho_{+}, \rho_{-}}(h)$. Thus, the following inequality holds for any fixed $\rho_{-} > 0$:
\begin{equation}
\label{eq:aux}
\begin{aligned}
\Pr \bracket[\bigg]{\sup_{\substack{h \in \sH\\ \rho_{+} \in (0, r_{+}]}} \sR_{\ell_{0-1}}(h) - \h \sR_{S}^{\rho_{+}, \rho_{-}}(h) - 2 \Rad_{m}^{\rho_{+}/2, \rho_{-}}(\sH) - \sqrt{\frac{\log \log_2 (2 r_{+} / \rho_{+})}{m}} - \e > 0} \leq 2 e^{-2m \e^2}.
\end{aligned}
\end{equation}
Next, consider two sequences $\paren*{\rho_{-}^l}_{l \geq 1}$ and $\paren*{\e_l}_{l \geq 1}$, with $\e_l \in (0, 1]$. By inequality \eqref{eq:aux}, for any fixed $l \geq 1$,
\begin{align*}
\Pr \bracket[\bigg]{\sup_{\substack{h \in \sH\\ \rho_{+} \in (0, r_{+}]}} \sR_{\ell_{0-1}}(h) - \h \sR_{S}^{\rho_{+}, \rho_{-}^l}(h) - 2 \Rad_{m}^{\rho_{+}/2, \rho_{-}^l}(\sH) - \sqrt{\frac{\log \log_2 (2 r_{+} / \rho_{+})}{m}} - \e_l > 0} \leq 2 e^{-2m \e_l^2}.
\end{align*}
Choosing $\e_l = \e + \sqrt{\frac{\log l}{m}}$, then, by the union bound, the following holds:
\begin{align*}
& \Pr \bracket[\bigg]{\sup_{\substack{h \in \sH\\ \rho_{+} \in (0, r_{+}] \\ l \geq 1}} \sR_{\ell_{0-1}}(h) - \h \sR_{S}^{\rho_{+}, \rho_{-}^l}(h) - 2 \Rad_{m}^{\rho_{+}/2, \rho_{-}^l}(\sH) - \sqrt{\frac{\log \log_2 (2 r_{+} / \rho_{+})}{m}} - \e_l > 0}\\
& \qquad \leq \sum_{l \geq 1} 2 e^{-2m \e_l^2} = 2 \sum_{l \geq 1} \exp^{-2m \paren*{\e + \sqrt{\frac{\log l}{m}}}^2} \leq 2 \sum_{l \geq 1} e^{-2m \e^2} e^{-2\log l} = 2 \paren*{\sum_{l \geq 1} 1 / l^2} e^{-2m \e^2} \leq 4 e^{-2m \e^2}.
\end{align*}
We can choose $\rho_{-}^l = r_{-} / 2^l$. For any $\rho_{-} \in (0, r_{-} ]$, there exists $l \geq 1$ such that $\rho_{-} \in (\rho_{-}^l, \rho_{-}^{l - 1}]$, with $\rho_{-}^0 = r_{-}$. For that $l$, $\rho_{-} \leq \rho_{-}^{l - 1} = 2 \rho_{-}^l$, thus $1 / \rho_{-}^l \leq 2 / \rho_{-}$ and $\sqrt{\log l} = \sqrt{\log \log_2 (r_{-} / \rho_{-}^l)} \leq \sqrt{\log \log_2 (2 r_{-} / \rho_{-})}$. Furthermore, for any $h \in \sH$, $\h \sR_{S}^{\rho_{+}, \rho_{-}^l}(h) \leq \h \sR_{S}^{\rho_{+}, \rho_{-}}(h)$. Thus, the following inequality holds:
\begin{align*}
\Pr \bracket[\bigg]{& \sup_{\substack{h \in \sH\\ \rho_{+} \in (0, r_{+}] \\ \rho_{-} \in (0, r_{-}]}} \sR_{\ell_{0-1}}(h) - \h \sR_{S}^{\rho_{+}, \rho_{-}}(h) - 4 \Rad_{m}^{\rho_{+}, \rho_{-}}(\sH) - \sqrt{\frac{\log \log_2 (2 r_{+} / \rho_{+})}{m}} - \sqrt{\frac{\log \log_2 (2 r_{-} / \rho_{-})}{m}} - \e > 0}  \leq 4 e^{-2m \e^2},
\end{align*}
where we used the fact that $\Rad_{m}^{\rho_{+}/2, \rho_{-}/2}(\sH) = 2 \Rad_{m}^{\rho_{+}, \rho_{-}}(\sH)$.
This proves the first statement. The second statement can be proven in a similar way.
\end{proof}

\subsection{Linear Hypotheses}
\label{app:linear}

Combining Theorem~\ref{thm:rad-linear} and
Theorem~\ref{thm:margin-bound-binary} gives directly the following
general margin bound for linear hypotheses with bounded weighted
vectors.
\begin{corollary}
\label{cor:margin-bound-binary-lin}
Let $\sH = \curl*{x \mapsto w \cdot x \colon \norm*{w} \leq \Lambda}$
and assume $\sX \subseteq \curl*{x \colon \norm*{x} \leq r}$. Let
$r_{+} = \sup_{i \in I_+} \norm*{x_i}$ and $r_{-} = \sup_{i \in I_{-}}
\norm*{x_i}$. Fix $\rho_{+} > 0$ and $\rho_{-} > 0$, then, for any
$\delta > 0$, with probability at least $1 - \delta$ over the choice
of a sample $S$ of size $m$, the following holds for any $h \in \sH$:
\begin{align*}
\sR_{\ell_{0-1}}(h) 
& \leq \h \sR_{S}^{\rho_{+}, \rho_{-}}(h) + \frac{2 \Lambda}m \sqrt{\frac{m_{+} r_{+}^2}{\rho_{+}^2} + \frac{m_{-} r_{-}^2}{\rho_{-}^2}} + 3 \sqrt{\frac{\log \frac{2}{\delta}}{2m}}\\
& \leq \h \sR_{S}^{\rho_{+}, \rho_{-}}(h) + \frac{2 \Lambda r}{m} \sqrt{\frac{m_{+}}{\rho_{+}^2} + \frac{m_{-}}{\rho_{-}^2}} + 3 \sqrt{\frac{\log \frac{2}{\delta}}{2m}}.
\end{align*}
\end{corollary}
Choosing $\Lambda = 1$, by the generalization of
Corollary~\ref{cor:margin-bound-binary-lin} to a uniform bound over
$\rho_{+} \in (0, r_{+}]$ and $\rho_{-} \in (0, r_{-}]$, for any
    $\delta > 0$, with probability at least $1 - \delta$, the
    following holds for all $h \in \curl*{x \mapsto w \cdot x \colon
      \norm*{w} \leq 1}$, $\rho_{+} \in (0, r_{+}]$ and $\rho_{-} \in
      (0, r_{-}]$:
\begin{equation}
\sR_{\ell_{0-1}}(h) \leq \h \sR_{S}^{\rho_{+}, \rho_{-}}(h)  + \frac{4 r}{m} \sqrt{\frac{m_{+}}{\rho_{+}^2} + \frac{m_{-}}{\rho_{-}^2}} + \sqrt{\frac{\log \log_2 \frac{2r_{+}}{\rho_{+}}}{m}} + \sqrt{\frac{\log \log_2 \frac{2r_{-}}{\rho_{-}}}{m}}  + 3\sqrt{\frac{\log \frac{8}{\delta}}{2m}}.
\end{equation}
\ignore{The inequality also trivially holds for $\rho_{+}$ larger than
  $r_{+}$ or $\rho_{-}$ larger than $r_{-}$, since in that case, by
  the Cauchy-Schwarz inequality, for any $w$ with $\norm*{w} \leq 1$,
  we have $\sup_{i \in I_+ }y_i (w \cdot x_i) \leq r_{+} \leq
  \rho_{+}$ and $\sup_{i \in I_{-} }y_i (w \cdot x_i) \leq r_{-} \leq
  \rho_{-}$. Thus, $\h \sR_{S}^{\rho_{+}, \rho_{-}}(h)$ is equal to
  one for all $h$.}

Now, for any $\rho > 0$, the
$\rho$-margin loss function is upper bounded by the $\rho$-hinge loss:
\begin{equation*}
  \forall u \in \Rset,\quad \Phi_{\rho}(u)
  = \min \paren*{1, \max \paren*{0, 1 - \frac{u}{\rho}}}
  \leq \max \paren*{0, 1 - \frac{u}{\rho}}.
\end{equation*}
Thus, with probability at least $1 - \delta$, the following holds for
all $h \in \curl*{x \mapsto w \cdot x \colon \norm*{w} \leq 1}$,
$\rho_{+} \in (0, r_{+}]$ and $\rho_{-} \in (0, r_{-}]$:
\begin{equation}
\begin{aligned}
\sR_{\ell_{0-1}}(h) & \leq \frac{1}{m}  \bracket*{ \sum_{i \in I_+} \max \paren*{0, 1- \frac{y_i h(x_i)}{\rho_{+}}} + \sum_{i \in I_-} \max \paren*{0, 1 - \frac{y_i h(x_i)}{\rho_{-}}}}\\
& \qquad + \frac{4 r}{m} \sqrt{\frac{m_{+}}{\rho_{+}^2} + \frac{m_{-}}{\rho_{-}^2}}  + \sqrt{\frac{\log \log_2 \frac{2r_{+}}{\rho_{+}}}{m}} + \sqrt{\frac{\log \log_2 \frac{2r_{-}}{\rho_{-}}}{m}} + \sqrt{\frac{\log \frac{4}{\delta}}{2m}}.
\end{aligned}
\end{equation}
Since for any $\rho > 0$, $h / \rho$ admits the same generalization
error as $h$, with probability at least $1 - \delta$, the following
holds for all $h \in \curl*{x \mapsto w \cdot x \colon \norm*{w} \leq
  \frac{1}{\rho_{+} + \rho_{-}}}$, $\rho_{+} \in (0, r_{+}]$ and
  $\rho_{-} \in (0, r_{-}]$:
\begin{align*}
\sR_{\ell_{0-1}}(h) 
& \leq \frac{1}{m}  \bracket*{ \sum_{i \in I_+} \max \paren*{0, 1- y_i h(x_i) \paren*{\frac{\rho_{+}
+ \rho_{-}}{\rho_{+}}} } + \sum_{i \in I_-} \max \paren*{0, 1 - y_i h(x_i) \paren*{\frac{\rho_{+} + \rho_{-}}{\rho_{-}}}}}\\
&\qquad  + \frac{4 r}{m} \sqrt{\frac{m_{+}}{\rho_{+}^2} + \frac{m_{-}}{\rho_{-}^2}} + \sqrt{\frac{\log \log_2 \frac{2r_{+}}{\rho_{+}}}{m}} + \sqrt{\frac{\log \log_2 \frac{2r_{-}}{\rho_{-}}}{m}}  + \sqrt{\frac{\log \frac{4}{\delta}}{2m}}.
\end{align*}
Now, since only the first term of the right-hand side depends on $w$,
the bound suggests selecting $w$ as the solution of the following
optimization problem:
\begin{align*}
\min_{\norm*{w}^2 \leq \paren*{\frac{1}{\rho_{+} + \rho_{-}}}^2} \frac{1}{m}  \bracket*{ \sum_{i \in I_+} \max \paren*{0, 1- y_i h(x_i) \paren*{\frac{\rho_{+} + \rho_{-}}{\rho_{+}}} } + \sum_{i \in I_-} \max \paren*{0, 1 - y_i h(x_i) \paren*{\frac{\rho_{+} + \rho_{-}}{\rho_{-}}}}}.
\end{align*}
Introducing a Lagrange variable $\lambda \geq 0$ and a free variable
$\alpha = \frac{\rho_{+}}{\rho_{+} + \rho_{-}} > 0$, the optimization
problem can be written equivalently as
\begin{equation}
\label{eq:alg}
\min_{w} \lambda \norm*{w}^2 +  \frac{1}{m}  \bracket*{ \sum_{i \in I_+} \max \paren*{0, 1- y_i \frac{w \cdot x_i}{\alpha}} + \sum_{i \in I_-} \max \paren*{0, 1 - y_i \frac{w \cdot x_i}{1 - \alpha}}},   
\end{equation}
where $\lambda$ and $\alpha$ can be selected via cross-validation. The
resulting algorithm can be viewed as an extension of SVMs.

Note that while $\alpha$ can be freely searched over different values,
we can search near the optimal values found in the separable case in
\eqref{eq:optimal-separable}. Also, the solution can actually be
obtained using regular SVM by incorporating the $\alpha$ multipliers
into the feature vectors. Furthermore, we can replace the hinge loss
with a general margin-based loss function $\Psi \colon u \mapsto
\Rset_{+}$, and we can add a bias term $b > 0$ for the linear models
if the data is not normalized:
\begin{align}
\label{eq:alg-general}
\min_{w, b} \lambda \norm*{w}^2
+ \frac{1}{m}  \bracket*{ \sum_{i \in I_+} \Psi \paren*{y_i \frac{w \cdot x_i + b}{\alpha}}
  + \sum_{i \in I_-} \Psi \paren*{y_i \frac{w \cdot x_i + b}{1 - \alpha}}},
\end{align}
For example, $\Psi$ can be chosen as the logistic loss function $u
\mapsto \log_2(1 + e^{-u})$ or the exponential loss function $u
\mapsto e^{-u}$.

\subsection{Proof of Theorem~\ref{thm:rad-linear}}
\label{app:rad-linear}

\RadLinear*
\begin{proof}
The proof follows through a series of inequalities:
\begin{align*}
& \h \Rad_{S}^{\rho_{+}, \rho_{-}}(\sH)\\
& = \frac{1}{m} \E_{\sigma}\bracket*{\sup_{\norm*{w} \leq \Lambda} w \cdot \paren*{ \frac{1}{\rho_{+}} \paren*{\sum_{i \in I_+} \sigma_i x_i} + \frac{1}{\rho_{-}} \paren*{\sum_{i \in I_{-}} -\sigma_i x_i}}}\\
& \leq \frac{\Lambda}m \E_{\sigma}\bracket*{\norm*{ \frac{1}{\rho_{+}} \paren*{\sum_{i \in I_+} \sigma_i x_i} + \frac{1}{\rho_{-}} \paren*{\sum_{i \in I_{-}} -\sigma_i x_i}}} \leq \frac{\Lambda}m \bracket*{\E_{\sigma}\bracket*{\norm*{ \frac{1}{\rho_{+}} \paren*{\sum_{i \in I_+} \sigma_i x_i} + \frac{1}{\rho_{-}} \paren*{\sum_{i \in I_{-}} -\sigma_i x_i}}^2}}^{\frac12}\\ 
& \leq \frac{\Lambda}m \bracket*{\frac{1}{\rho_{+}^2} \sum_{i \in I_+} \norm*{x_i}^2 + \frac{1}{\rho_{-}^2} \sum_{i \in I_{-}} \norm*{x_i}^2}^{\frac12} \leq \frac{\Lambda}m \sqrt{\frac{m_{+} r_{+}^2}{\rho_{+}^2} + \frac{m_{-} r_{-}^2}{\rho_{-}^2}}  \leq \frac{\Lambda r}{m} \sqrt{\frac{m_{+}}{\rho_{+}^2} + \frac{m_{-}}{\rho_{-}^2}}.
\end{align*}
The first inequality makes use of the Cauchy-Schwarz inequality and the bound on $\norm*{w}$, the second follows by Jensen's inequality, the third by $\E[\sigma_i \sigma_j] = \E[\sigma_i] \E[\sigma_j] = 0$ for $i \neq j$, the fourth by $\sup_{i \in I_+} \norm*{x_i} = r_{+}$ and $\sup_{i \in I_{-}} \norm*{x_i} = r_{-}$, and the last one by $\norm*{x_i} \leq r$.
\end{proof}

\section{Extension to Multi-Class Classification}
\label{app:multiclass}

In this section, we extend the previous analysis and algorithm to
multi-class classification. We will adopt the same notation and
definitions as previously described, with some slight adjustments. In
particular, we denote the multi-class label space by $\sY = [c]
\coloneqq \curl*{1, \ldots, c}$ and a hypothesis set of functions
mapping from $\sX \times \sY$ to $\Rset$ by $\sH$. For a hypothesis $h
\in \sH$, the label $\hh(x)$ assigned to $x \in \sX$ is the one with
the largest score, defined as $\hh(x) = \argmax_{y \in \sY} h(x, y)$,
using the highest index for tie-breaking. For a labeled example $(x,
y) \in \sX \times \sY$, the \emph{margin} $\rho_h(x, y)$ of a
hypothesis $h \in \sH$ is given by $\rho_h(x, y) = h(x, y) - \max_{y'
  \neq y} h(x, y')$, which is the difference between the score
assigned to $(x, y)$ and that of the next-highest scoring label. We
define the multi-class zero-one loss function as
$\ell^{\rm{multi}}_{0-1} \coloneqq \1_{\hh(x) \neq y}$. This is the
target loss of interest in multi-class classification.

\subsection{Multi-Class Imbalanced Margin Loss}
\label{sec:imbalanced-loss-multi}

We first extend the class-imbalanced margin loss function to the multi-class
setting. To account for different confidence margins for instances
with different labels, we define the \emph{multi-class class-imbalanced
margin loss function} as follows:
\begin{definition}[Multi-class class-imbalanced margin loss]
For any $\brho = [\rho_k]_{k \in [c]}$, the \emph{multi-class
class-imbalanced $\brho$-margin loss} is the function $\sfL_{\brho} \colon
\sH_{\mathrm{all}} \times \sX \times \sY \to \Rset$, defined as
follows:
\begin{equation}
  \sfL_{\brho}(h, x, y) = \sum_{k = 1}^c
  \Phi_{\rho_k}\paren*{\rho_h(x, y)} \1_{y  = k}.
\end{equation}
\end{definition}
The main margin bounds in this section are expressed in terms of this
loss function. The parameters $\rho_k > 0$, for $k \in [c]$, represent
the confidence margins imposed by a hypothesis $h$ for instances
labeled $k$.  The following result provides an equivalent expression
for the class-imbalanced margin loss function. The proof is included in
Appendix~\ref{app:margin-loss-multi}.
\MarginLossMulti*

\subsection{\texorpdfstring{$\sH$}{H}-Consistency}
\label{sec:H-consistency-multi}

The following result provides a strong consistency guarantee for the
multi-class class-imbalanced margin loss introduced in relation to the
multi-class zero-one loss. We say a hypothesis set is complete when
the scoring values spanned by $\sH$ for each instance cover $\Rset$:
for all $(x, y) \in \sX \times \sY$, $\curl*{h(x, y) \colon h \in \sH}
= \Rset$.

\HConsistencyMulti*
The proof is included in Appendix~\ref{app:H-consistency-multi}. The
next section presents generalization bounds based on the empirical
multi-class class-imbalanced margin loss, along with the
\emph{$\brho$-class-sensitive Rademacher complexity} and its empirical
counterpart defined below. Given a sample $S = \paren*{x_1, \ldots,
  x_m}$, for any $k \in [c]$, we define $I_k = \curl*{i \in \curl*{1,
    \ldots, m} \mid y_i = k}$ and $m_{k} = |I_k|$ as the number of
instances labeled $k$.

\begin{definition}[$\brho$-class-sensitive Rademacher complexity]
Let $\sH$ be a family of functions mapping from $\sX \times \sY$ to
$\Rset$ and $S = \paren*{(x_1, y_1) \ldots, (x_m, y_m)}$ a fixed
sample of size $m$ with elements in $\sX \times \sY$. Fix $\brho =
[\rho_k]_{k \in [c]} > \mathbf{0}$. Then, the \emph{empirical
$\brho$-class-sensitive Rademacher complexity of $\sH$} with respect to
the sample $S$ is defined as:
\begin{equation}
  \h \Rad_{S}^{\brho}(\sH)
  = \frac{1}{m} \E_{\e}\bracket*{\sup_{h \in \sH}
    \curl*{ \sum_{k = 1}^c \sum_{i \in I_k}
      \sum_{y \in \sY} \e_{iy} \frac{h(x_i, y)}{\rho_{k}} }},
\end{equation}
where $\e = \paren*{\e_{i y}}_{i, y}$ with $\e_{iy}$s being
independent variables uniformly distributed over $\curl*{-1, +1}$.
For any integer $m \geq 1$, the \emph{$\brho$-class-sensitive
Rademacher complexity of $\sH$} is the expectation of the empirical
$\brho$-class-sensitive Rademacher complexity over all samples
of size $m$ drawn according to $\sD$: $\Rad_{m}^{\brho}(\sH) = \E_{S
  \sim \sD^m} \bracket*{\h \Rad_{S}^{\brho}(\sH)}$.
\end{definition}

\subsection{Margin-Based Guarantees}
\label{sec:margin-bound-multi}

Next, we will prove a general margin-based generalization bound, which
will serve as the foundation for deriving new algorithms for
imbalanced multi-class classification.

Given a sample $S = \paren*{x_1, \ldots, x_m}$ and a hypothesis $h$,
the \emph{empirical multi-class class-imbalanced margin loss} is defined by
$\h \sR_{S}^{\brho}(h) = \frac{1}{m} \sum_{i = 1}^m \sfL_{\brho}(h,
x_i, y_i)$. Note that the multi-class zero-one loss function
$\ell^{\rm{multi}}_{0-1}$ is upper bounded by the multi-class
class-imbalanced margin loss $\sfL_{\brho}$: $
\sR_{\ell^{\rm{multi}}_{0-1}}(h) \leq \sR_{\sfL_{\brho}}(h).  $

\begin{restatable}[Margin bound for imbalanced multi-class classification]
  {theorem}{MarginBoundMulti}
\label{thm:margin-bound-multi}
Let $\sH$ be a set of real-valued functions. Fix $\rho_{k} > 0$ for $k
\in [c]$, then, for any $\delta > 0$, with probability at least $1 -
\delta$, each of the following holds for all $h \in \sH$:
\begin{align*}
  \sR_{\ell^{\rm{multi}}_{0-1}}(h) &\leq \h \sR_{S}^{\brho}(h)
  + 4 \sqrt{2c}\, \Rad_{m}^{\brho}(\sH) + \sqrt{\frac{\log \frac{1}{\delta}}{2m}}\\
  \sR_{\ell^{\rm{multi}}_{0-1}}(h) &\leq \h \sR_{S}^{\brho}(h)
  + 4 \sqrt{2c}\, \h \Rad_{S}^{\brho}(\sH) + 3 \sqrt{\frac{\log \frac{2}{\delta}}{2m}}.
\end{align*}
\end{restatable}
The proof is presented in Appendix~\ref{app:margin-bound-multi}.  As
in Theorem~\ref{thm:margin-bound-binary-uniform}, these bounds can be
generalized to hold uniformly for all $\rho_k \in (0, 1]$, at the cost
  of additional terms $\sqrt{\frac{\log \log_2
      \frac{2}{\rho_{k}}}{m}}$ for $k \in [c]$, as shown in
  Theorem~\ref{thm:margin-bound-multi-uniform} in
  Appendix~\ref{app:margin-bound-multi-uniform}.

As for margin bounds in imbalanced binary classification, they show
the conflict between two terms: the larger the desired margins
$\brho$, the smaller the second term, at the price of a larger
empirical multi-class class-imbalanced margin loss $\h
\sR_{S}^{\brho}$. Note, however, that here there is additionally a
dependency on the number of classes $c$. This suggests either weak
guarantees when learning with a large number of classes or the need
for even larger margins $\brho$ for which the empirical multi-class
class-imbalanced margin loss would be small.

\subsection{General Multi-Class Classification Algorithms}
\label{sec:algorithms-multi}

Here, we derive \IMMAX\  algorithms for multi-class classification in
imbalanced settings, building on the theoretical analysis from the
previous section.

Let $\Phi$ be a feature mapping from $\sX \times \sY$ to
$\Rset^{d}$. Let $S \subseteq \curl*{(x, y) \colon \norm*{\Phi(x, y)}
  \leq r}$ denote a sample of size $m$, for some appropriate norm
$\norm*{\, \cdot \,}$ on $\Rset^d$. Define $r_{k} = \sup_{i \in I_k, y
  \in \sY} \norm*{\Phi(x_i, y)}$, for any $k \in [c]$. As in the
binary case, we assume that the empirical class-sensitive Rademacher
complexity $\h \Rad_{S}^{\brho}(\sH)$ can be bounded as:
\begin{equation*}
  \h \Rad_{S}^{\brho}(\sH)
  \leq \frac{\Lambda_{\sH} \sqrt{c}}{m} \sqrt{\sum_{k = 1}^c \frac{m_k r_{k}^2}{\rho_k^2}}
  \leq \frac{\Lambda_{\sH} r \sqrt{c}}{m} \sqrt{\sum_{k = 1}^c \frac{m_k}{\rho_k^2}},
\end{equation*}
where $\Lambda_{\sH}$ depends on the complexity of the hypothesis set
$\sH$. This bound holds for many commonly used hypothesis sets.  For a
family of neural networks, $\Lambda_\sH$ can be expressed as a
Frobenius norm
\citep{CortesGonzalvoKuznetsovMohriYang2017,NeyshaburTomiokaSrebro2015}
or spectral norm complexity with respect to reference weight matrices
\cite{BartlettFosterTelgarsky2017}. Additionally,
Theorems~\ref{thm:rad-kernel-1} and \ref{thm:rad-kernel-2} in
Appendix~\ref{app:linear-multi} address kernel-based hypotheses. More
generally, for the analysis that follows, we will assume that $\sH$
can be defined by $\sH = \curl*{h \in \ov \sH \colon \norm{h} \leq
  \Lambda_\sH}$, for some appropriate norm $\norm*{\, \cdot \,}$ on
some space $\ov \sH$.  Combining such an upper bound and
Theorem~\ref{thm:margin-bound-multi} or
Theorem~\ref{thm:margin-bound-multi-uniform}, gives directly the
following general margin bound:
\begin{align*}
  \sR_{\ell^{\rm{multi}}_{0-1}}(h) &\leq \h \sR_{S}^{\brho}(h)
  + \frac{4 \sqrt{2} \Lambda_{\sH} r  c }{m}
  \sqrt{\sum_{k = 1}^c \frac{m_k}{\rho_k^2}} + O\paren*{\frac{1}{\sqrt{m}}},
\end{align*}
where the last term includes the $\log$-$\log$ terms and the
$\delta$-confidence term. Let $\Psi$ be a non-increasing convex
function such that $\Phi_{\rho}(u) \leq
\Psi\left(\frac{u}{\rho}\right)$ for all $u \in \Rset$. Then, since
$\Phi_\rho$ is non-increasing, for any $(x, k)$, we have:
$\Phi_{\rho}(\rho_h(x, k))
 = \max_{j \neq k} \Phi_{\rho}(h(x, k) - h(x, j)).$
This
suggests a regularization-based algorithm of the following form:
\begin{equation}
\min_{h \in \ov \sH} \lambda \norm*{h}^2
+ \frac{1}{m} \bracket*{\sum_{k = 1}^c
  \sum_{i \in I_k} \max_{j \neq k}
  \Psi \paren*{\tfrac{h(x, k) - h(x, j)}{\rho_k}}},
\end{equation}
where $\lambda$ and $\rho_k$s are chosen via cross-validation.  In
particular, choosing $\Psi$ to be the logistic loss and upper-bounding
the maximum by a sum yields the following form for our
\IMMAX\ (\emph{Imbalanced Margin Maximization}) algorithm:
\begin{equation}
\min_{h \in \ov \sH} \lambda \norm*{h}^2
+ \frac{1}{m} \sum_{k = 1}^c
\sum_{i \in I_k} \mspace{-2mu} \log \bracket*{\sum_{j = 1}^c
  \exp\paren*{\tfrac{h(x_i, j) - h(x_i, k)}{\rho_{k}}}},
\end{equation}
where $\lambda$ and $\rho_k$s are chosen via cross-validation.
Let $\rho = \sum_{k = 1}^c \rho_k$ and $\ov r = \bracket*{\sum_{k =
    1}^c m_k^{\frac{1}{3}} r_{k, 2}^{\frac{2}{3}}}^{\frac{3}{2}}$.
Using Lemma~\ref{lemma:D3} (Appendix~\ref{app:lemma}), the expression
under the square root in the second term of the generalization bound
can be reformulated in terms of the R\'enyi divergence of order 3 as:
$\sum_{k = 1}^c \frac{m_k r_{k, 2}^2}{\rho_k^2} = \frac{\ov
  r^2}{\rho^2} e^{2 \sfD_3\paren*{\sfr \, \| \, \frac{\brho}{\rho}
}}$, where $\sfr = \bracket[\bigg]{ \frac{m_k^{\frac{1}{3}} r_{k,
      2}^{\frac{2}{3}}} {\ov r^{\frac{2}{3}}}}_{k}$.  Thus, while
$\rho_k$s can be freely searched over a range of values in our general
algorithm, it may be beneficial to focus the search for the vector
$[\rho_k/\rho]_k$ near $\sfr$. This strictly generalizes our binary
classification results and the analysis of the separable case.

When the number of classes $c$ is very large, the search space can be
further reduced by constraining the $\rho_k$ values for
underrepresented classes to be identical and allowing distinct
$\rho_k$ values only for the most frequently occurring classes.

\section{Multi-Class Classification: Proofs}
\label{app:multi}

\subsection{Proof of Lemma~\ref{lemma:margin-loss-multi}}
\label{app:margin-loss-multi}

\MarginLossMulti*
\begin{proof}
When $\rho_h(x, y) \leq 0$, we have $\Phi_{\rho_k}\paren*{\rho_h(x, y)} = 1$ for any $k \in [c]$, so the equality holds. When $\rho_h(x, y) > 0$, we have $y = k \iff \rho_h(x, k) > 0 \iff \hh(x) = k$, which also implies the equality.
\end{proof}

\subsection{Proof of Theorem~\ref{thm:H-consistency-multi}}
\label{app:H-consistency-multi}

\HConsistencyMulti*
\begin{proof}
Let $p(y \!\mid\! x) = \mathbb{P}(Y = y \!\mid\! X = x)$ denote the
conditional probability that $Y = y$ given $X = x$. Then, the
conditional error and the best-in-class conditional error of the
zero-one loss can be expressed as follows:
\begin{align*}
\E_{y} \bracket*{\ell^{\rm{multi}}_{0-1}(h, x, y) \mid x} &=  \sum_{y \in \sY} p(y \!\mid\! x) \1_{\hh(x) \neq y} = 1 - p(\hh(x) \!\mid\! x),\\
\inf_{h \in \sH }\E_{y} \bracket*{\ell^{\rm{multi}}_{0-1}(h, x, y) \mid x} &= 1 - \max_{y \in \sY} p(y \!\mid\! x).
\end{align*}
Furthermore, the difference between the two terms is given by:
\begin{align*}
\E_{y} \bracket*{\ell^{\rm{multi}}_{0-1}(h, x, y) \mid x} - \inf_{h \in \sH }\E_{y} \bracket*{\ell^{\rm{multi}}_{0-1}(h, x, y) \mid x} = \max_{y \in \sY} p(y \!\mid\! x) - p(\hh(x) \!\mid\! x).
\end{align*}
For the multi-class class-imbalanced margin loss, the conditional error can be expressed as follows:
\begin{align*}
\E_{y} \bracket*{\sfL_{\brho}(h, x, y) \mid x} &=  \sum_{y \in \sY} p(y \!\mid\! x) \Phi_{\rho_y}(\rho_h(x, y)) \\
& =  \sum_{y \in \sY} p(y \!\mid\! x) \min \paren*{1, \max \paren*{0, 1 - \frac{\rho_h(x, y)}{\rho_y}}}\\
& = 1  - p(\hh(x) \!\mid\! x) + p(\hh(x) \!\mid\! x) \max \paren*{0, 1 - \frac{\rho_h(x, \hh(x))}{\rho_{\hh(x)}}}\\
& = 1 - p(\hh(x) \!\mid\! x) \min \paren*{1, \frac{\rho_h(x, \hh(x))}{\rho_{\hh(x)}}}.
\end{align*}
Thus, the best-in-class conditional error can be expressed as follows:
\begin{equation*}
\inf_{h \in \sH } \E_{y} \bracket*{\sfL_{\brho}(h, x, y) \mid x} = 1 - \max_{y \in \sY} p(y \!\mid\! x).
\end{equation*}
The difference between the two terms is given by:
\begin{align*}
\E_{y} \bracket*{\sfL_{\brho}(h, x, y) \mid x} - \inf_{h \in \sH }\E_{y} \bracket*{\sfL_{\brho}(h, x, y) \mid x} &= \max_{y \in \sY} p(y \!\mid\! x) - p(\hh(x) \!\mid\! x) \min \paren*{1, \frac{\rho_h(x, \hh(x))}{\rho_{\hh(x)}}}\\
& \geq \max_{y \in \sY} p(y \!\mid\! x) - p(\hh(x) \!\mid\! x)\\
& = \E_{y} \bracket*{\ell^{\rm{multi}}_{0-1}(h, x, y) \mid x} - \inf_{h \in \sH }\E_{y} \bracket*{\ell^{\rm{multi}}_{0-1}(h, x, y) \mid x}.
\end{align*}
By taking the expectation of both sides, we obtain:
\begin{equation*}
\sR_{\ell^{\rm{multi}}_{0-1}}(h) - \sR^*_{\ell^{\rm{multi}}_{0-1}}(\sH) + \sM_{\ell^{\rm{multi}}_{0-1}}(\sH) \leq \sR_{\sfL_{\brho}}(h) - \sR^*_{\sfL_{\brho}}(\sH) + \sM_{\sfL_{\brho}}(\sH),
\end{equation*}
which completes the proof.
\end{proof}

\subsection{Proof of Theorem~\ref{thm:margin-bound-multi}}
\label{app:margin-bound-multi}

\begin{proof}
Consider the family of functions taking values in $[0, 1]$:
\begin{equation*}
\wt \sH = \curl*{z = (x, y) \mapsto \sfL_{\brho}(h, x, y) \colon h \in \sH}.
\end{equation*}
By \citep[Theorem~3.3]{MohriRostamizadehTalwalkar2018}, with probability at least $1 - \delta$, for all $g \in \wt \sH$,
\begin{equation*}
\E[g(z)] \leq \frac{1}{m} \sum_{i = 1}^m g(z_i) + 2 \h \Rad_{S}(\wt \sH) + 3 \sqrt{\frac{\log \frac{2}{\delta}}{2m}},
\end{equation*}
and thus, for all $h \in \sH$,
\begin{equation*}
\E[\sfL_{\brho}(h, x, y)] \leq \h \sR_{S}^{\brho}(h) + 2 \h \Rad_{S}(\wt \sH) + 3 \sqrt{\frac{\log \frac{2}{\delta}}{2m}}.    
\end{equation*}
Since $\sR_{\ell^{\rm{multi}}_{0-1}}(h) \leq \sR_{\sfL_{\brho}}(h) = \E[\sfL_{\brho}(h, x, y)]$, we have
\begin{equation*}
\sR_{\ell^{\rm{multi}}_{0-1}}(h) \leq \h \sR_{S}^{\brho}(h) + 2 \h \Rad_{S}(\wt \sH) + 3 \sqrt{\frac{\log \frac{2}{\delta}}{2m}}.
\end{equation*}
For convenience, we define $\rho(i) = \sum_{k = 1}^c \rho_{k} \1_{i \in I_k}$ for $ i = 1, \ldots, m$.
Since $\Phi_{\rho}$ is $\frac{1}{\rho}$-Lipschitz, by \citep[Lemma~5.7]{MohriRostamizadehTalwalkar2018}, $\h \Rad_{S}(\wt \sH)$ can be rewritten as follows:
\begin{align*}
\h \Rad_{S}(\wt \sH) &= \frac{1}{m} \E_{\sigma}\bracket*{\sup_{h \in \sH} \sum_{i = 1}^m \sigma_i \sfL_{\brho}(h, x_i, y_i)}\\
&= \frac{1}{m} \E_{\sigma}\bracket*{\sup_{h \in \sH} \sum_{i = 1}^m \sigma_i \bracket*{\sum_{k = 1}^c \Phi_{\rho_k}\paren*{\rho_h(x_i, y_i)} \1_{y_i = k}}}\\
&\leq \frac{1}{m} \E_{\sigma}\bracket*{\sup_{h \in \sH} \curl*{ \sum_{i = 1}^m \sigma_i \frac{\rho_h(x_i, y_i)}{\rho(i)}}}\\
& = \frac{1}{m} \E_{\sigma}\bracket*{\sup_{h \in \sH} \curl*{ \sum_{i = 1}^m \sigma_i \frac{h(x_i, y_i) - \max_{y' \neq y_i} h(x_i, y')}{\rho(i)}}}\\
& \leq \frac{1}{m} \E_{\sigma}\bracket*{\sup_{h \in \sH} \curl*{ \sum_{i = 1}^m \sigma_i \frac{h(x_i, y_i)}{\rho(i)} }} + \frac{1}{m} \E_{\sigma}\bracket*{\sup_{h \in \sH} \curl*{ \sum_{i = 1}^m \sigma_i \frac{\max_{y' \neq y_i} h(x_i, y')}{\rho(i)} }}.
\end{align*}
Now we bound the second term above. For any $i = 1, \ldots, m$, consider the mapping $\Psi_i \colon h \mapsto \frac{\max_{y' \neq y_i} h(x_i, y')}{\rho(i)}$. Then, for any $h, h' \in \sH$, we have
\begin{align*}
\abs*{\Psi_i(h) - \Psi_i(h')} & \leq \max_{y' \neq y_i} \frac{\abs*{h(x_i, y') - h'(x_i, y')}}{\rho(i)}\\
& \leq \frac{1}{\rho(i)} \sum_{y \in \sY}  \abs*{h(x_i, y) - h'(x_i, y)}\\
& \leq \frac{\sqrt{c}}{\rho(i)} \sqrt{\sum_{y \in \sY}  \abs*{h(x_i, y) - h'(x_i, y)}^2}.
\end{align*}
Thus, $\Psi_i$ is $\frac{\sqrt{c}}{\rho(i)}$-Lipschitz with respect to the $\norm*{\cdot}_2$ norm. Thus, by \citep[Lemma~5]{cortes2016structured},
\begin{align*}
\frac{1}{m} \E_{\sigma}\bracket*{\sup_{h \in \sH} \curl*{ \sum_{i = 1}^m \sigma_i \frac{\max_{y' \neq y_i} h(x_i, y')}{\rho(i)} }} 
& \leq \frac{\sqrt{2}}{m} \E_{\sigma}\bracket*{\sup_{h \in \sH} \curl*{ \sum_{i = 1}^m \sum_{y \in \sY} \sigma_{iy} \frac{\sqrt{c}}{\rho(i)} h(x_i, y) }}\\
& = \frac{\sqrt{2 c}}{m} \E_{\e}\bracket*{\sup_{h \in \sH} \curl*{ \sum_{k = 1}^c \sum_{i \in I_k} \sum_{y \in \sY} \e_{iy} \frac{h(x_i, y)}{\rho_{k}} }}\\
& = \sqrt{2c}\, \h \Rad_{S}^{\brho}(\sH).
\end{align*}
We can proceed similarly with the first term to obtain
\begin{equation*}
\frac{1}{m} \E_{\sigma}\bracket*{\sup_{h \in \sH} \curl*{ \sum_{i = 1}^m \sigma_i \frac{h(x_i, y_i)}{\rho(i)} }} \leq \sqrt{2c}\, \h \Rad_{S}^{\brho}(\sH).
\end{equation*}
Thus, $\h \Rad_{S}(\wt \sH)$ can be upper bounded as follows:
\begin{equation*}
\h \Rad_{S}(\wt \sH) \leq 2 \sqrt{2c}\, \h \Rad_{S}^{\brho}(\sH).
\end{equation*}
This proves the second inequality. The first inequality, can be derived in the same way by using the first inequality of \citep[Theorem~3.3]{MohriRostamizadehTalwalkar2018}.
\end{proof}

\subsection{Analysis of the Second Term in the Generalization Bound}
\label{app:lemma}

In this section, we analyze the second term of the generalization
bound in terms of the R\'enyi entropy of order 3.

Recall that the R\'enyi divergence of positive order $\alpha$ between
two distributions $\sfp$ and $\sfq$ with support $[c]$ is defined as:
\[
\sfD_\alpha(\sfp \, \| \, \sfq)
= \frac{1}{\alpha - 1} \log \bracket*{\sum_{k = 1}^c
  \sfp_k^\alpha \sfq_k^{1 - \alpha}},
\]
with the conventions $\frac{0}{0} = 0$ and $\frac{x}{0} = \infty$ for
$x > 0$. This definition extends to $\alpha \in \curl*{0, 1, \infty}$
by taking appropriate limits. In particular, $\sfD_1$ corresponds to
the relative entropy (KL divergence).

\begin{lemma}
\label{lemma:D3}
Let $\rho = \sum_{k = 1}^c \rho_k$ and $\ov r = \bracket*{\sum_{k =
    1}^c m_k^{\frac{1}{3}} r_{k, 2}^{\frac{2}{3}}}^{\frac{3}{2}}$.
Then, the following identity holds:
  \begin{align*}
    \sum_{k = 1}^c \frac{m_k r_{k, 2}^2}{\rho_k^2}
    = \frac{\ov r^2}{\rho^2} e^{2 \sfD_3\paren*{\sfr \, \| \, \frac{\brho}{\rho} }},
  \end{align*}
  where $\sfr = \bracket[\bigg]{
    \frac{m_k^{\frac{1}{3}} r_{k, 2}^{\frac{2}{3}}}
         {\ov r^{\frac{2}{3}}}}_{k \in [c]}$.
\end{lemma}
\begin{proof}
The expression can be rewritten as follows after putting $\frac{\ov
    r^2}{\rho^2} \sum_{k = 1}^c$ in factor:
\begin{align*}
  \sum_{k = 1}^c \frac{m_k r_{k, 2}^2}{\rho_k^2}
  & = \frac{\ov r^2}{\rho^2} \sum_{k = 1}^c 
\frac{\paren*{\frac{\sqrt{m_k} r_{k, 2}}{\ov r}}^2}
{\paren[\big]{\frac{\rho_k}{\rho}}^2} \\
  & = \frac{\ov r^2}{\rho^2} \sum_{k = 1}^c 
\frac{\paren*{\frac{m_k^{\frac{1}{3}} r_{k, 2}^{\frac{2}{3}}}{\ov r^{\frac{2}{3}}}}^3}
{\paren[\big]{\frac{\rho_k}{\rho}}^{3 - 1}} \\
& = \frac{\ov r^2}{\rho^2} \exp\curl*{2 \, \sfD_3\paren*{\bracket[\bigg]{
      \frac{m_k^{\frac{1}{3}} r_{k, 2}^{\frac{2}{3}}}{\ov r^{\frac{2}{3}}}}_{k \in [c]} \, \Bigg\| \, \bracket*{\frac{\rho_k}{\rho}}_{k \in [c]}}}.
\end{align*}
This completes the proof.
\end{proof}
The lemma suggests that for fixed $\rho$, choosing $[\rho_k / \rho]_k$
close to $\sfr$ tends to minimize the second term of the
generalization bound. Specifically, in the separable case where the
empirical margin loss is zero, this analysis provides guidance on
selecting $\rho_k$s. The optimal values in this scenario align with
those derived in the analysis of the separable binary case.

\ignore{
Let $\ov \rho^2 = \sum_{k = 1}^c \rho_{k}^2$ and $\ov r_2 = \sum_{k = 1}^c \sqrt{m_k} r_{k, 2}$ be two  normalization constants. Then, 
\begin{align*}
\frac{\ov \rho^2}{\ov r_2^2} \sum_{k = 1}^c \frac{m_k r_{k, 2}^2}{\rho_k^2} & = \sum_{k = 1}^c \frac{m_k \frac{r_{k, 2}^2}{\ov r_2^2}}{\rho_k^2 / 
\ov \rho^2}\\
& = 1 + \sum_{k = 1}^c \frac{\paren*{\sqrt{m_k} \frac{r_{k, 2}}{\ov r_2} - \rho_k^2 / \ov \rho^2}^2}{\rho_k^2 / 
\ov \rho^2}\\
& = 1 + \chi^2 \paren*{\paren*{\frac{\sqrt{m_k} r_{k, 2}}{\ov r_2}}_{k = 1}^c \, \Bigg\| \, \paren*{\frac{\rho_k^2}{ \ov \rho^2}}_{k = 1}^c}
\tag{Def. of $\chi^2$-divergence}.
\ignore{
  & = e^{D_2 \paren*{\paren*{\sqrt{m_k} r_{k, 2}/\ov r_2}_{k = 1}^c || \paren*{\rho_k^2 / \ov \rho^2}_{k = 1}^c}} \tag{Def. of Rényi-divergence of order 2}
}
\end{align*}
Note that this can be expressed in terms of the R\'enyi divergence of order 2, $\sfD_2$, since for any two distributions $P$ and $Q$, we have: $\sfD_2 \paren*{P \mid \mid Q} = \ln\paren*{1 + \chi^2(P, Q)}$. 
 $\sfD_2$ is an upper bound on the relative entropy $\sfD_1$. 
}

\subsection{Uniform Margin Bound for Imbalanced Multi-Class Classification}
\label{app:margin-bound-multi-uniform}

\begin{theorem}[Uniform margin bound for imbalanced multi-class classification]
\label{thm:margin-bound-multi-uniform}

Let $\sH$ be a set of real-valued functions. Fix $r_{k} > 0$ for $k \in [c]$. Then, for any $\delta > 0$, with probability at least $1 - \delta$, each of the following holds for all $h \in \sH$ and $\rho_{k} \in (0, r_{k}]$ with $k \in [c]$:
\begin{align*}
\sR_{\ell^{\rm{multi}}_{0-1}}(h) &\leq \h \sR_{S}^{\brho}(h) + 4c \sqrt{2c}\, \Rad_{m}^{\brho}(\sH) + \sum_{k = 1}^c \sqrt{\frac{\log \log_2 \frac{2r_{k}}{\rho_{k}}}{m}}+ \sqrt{\frac{\log \frac{2^c}{\delta}}{2m}}\\
\sR_{\ell^{\rm{multi}}_{0-1}}(h) &\leq \h \sR_{S}^{\brho}(h) + 4c \sqrt{2c}\, \h \Rad_{S}^{\brho}(\sH) + \sum_{k = 1}^c \sqrt{\frac{\log \log_2 \frac{2r_{k}}{\rho_{k}}}{m}} + 3 \sqrt{\frac{\log \frac{2^{c + 1}}{\delta}}{2m}}.
\end{align*}
\end{theorem}

\subsection{Kernel-Based Hypotheses}
\label{app:linear-multi}

For some hypothesis sets, a simpler upper bound can be derived for the
$\brho$-class-sensitive Rademacher complexity of $\sH$,
thereby making Theorems~\ref{thm:margin-bound-multi} and
\ref{thm:margin-bound-multi-uniform} more explicit. We will show this
for kernel-based hypotheses. Let $K \colon \sX \times \sX \to \Rset$
be a PDS kernel and let $\Phi \colon \sX \to \Hset$ be a feature
mapping associated to $K$. We consider kernel-based hypotheses with
bounded weight vector: $\sH_{p} = \curl*{(x, y) \mapsto w \cdot
  \Phi(x, y) \colon w \in \Rset^d, \norm*{w}_p \leq \Lambda_p}$, where
$\Phi(x, y) = \paren*{\Phi_1(x, y), \ldots, \Phi_d(x, y)}^{\top}$ is a
$d$-dimensional feature vector. A similar analysis can be extended to
hypotheses of the form $(x, y) \mapsto w_y \cdot \Phi(x, y)$, where $
\norm*{w_y}_p \leq \Lambda_p$, based on $c$ weight vectors $w_1,
\ldots, w_c \in \Rset^d$. The empirical 
$\brho$-class-sensitive Rademacher complexity of $\sH_p$ with $p = 1$
and $p = 2$ can be bounded as follows.
\begin{restatable}{theorem}{RadKernelOne}
\label{thm:rad-kernel-1}
Consider $\sH_1 = \curl[\big]{(x, y) \mapsto w \cdot \Phi(x, y) \colon
  w \in \Rset^d, \norm*{w}_1 \leq \Lambda_1}$. Let $r_{k, \infty} =
\sup_{i \in I_k, y \in \sY} \norm{\Phi(x_i, y)}_{\infty}$, for any $k
\in [c]$.  Then, the following bound holds for all $h \in \sH$:
\begin{equation*}
  \h \Rad_{S}^{\brho}(\sH_1)
  \leq \frac{\Lambda_1 \sqrt{2c}}{m}
  \sqrt{\sum_{k = 1}^c \frac{m_k r_{k, \infty}^2}{\rho_k^2} \log(2d)}.
\end{equation*}
\end{restatable}

\begin{restatable}{theorem}{RadKernelTwo}
\label{thm:rad-kernel-2}
Consider $\sH_2 = \curl[\big]{(x, y) \mapsto w \cdot \Phi(x, y) \colon
  w \in \Rset^d, \norm*{w}_2 \leq \Lambda_2}$. Let $r_{k, 2} = \sup_{i
  \in I_k, y \in \sY} \norm*{\Phi(x_i, y)}_2$, for any $k \in
[c]$. Then, the following bound holds for all $h \in \sH$:
\begin{equation*}
  \h \Rad_{S}^{\brho}(\sH_2)
  \leq \frac{\Lambda_2 \sqrt{c}}{m} \sqrt{\sum_{k = 1}^c \frac{m_k r_{k, 2}^2}{\rho_k^2}}.
\end{equation*}
\end{restatable}

The proofs of Theorems~\ref{thm:rad-kernel-1}
and~\ref{thm:rad-kernel-2} are included in
Appendix~\ref{app:rad-kernel}. Combining
Theorem~\ref{thm:rad-kernel-1} or Theorem~\ref{thm:rad-kernel-2} with
Theorem~\ref{thm:margin-bound-multi} directly gives the following
general margin bounds for kernel-based hypotheses with bounded
weighted vectors, respectively.
\begin{corollary}
\label{cor:margin-bound-multi-kernel-1}
Consider $\sH_1 = \curl[\big]{(x, y) \mapsto w \cdot \Phi(x, y) \colon
  w \in \Rset^d, \norm*{w}_1 \leq \Lambda_1}$. Let $r_{k, \infty} =
\sup_{i \in I_k, y \in \sY} \norm*{\Phi(x_i, y)}_{\infty}$, for any $k
\in [c]$.  Fix $\rho_{k} > 0$ for $k \in [c]$, then, for any $\delta >
0$, with probability at least $1 - \delta$ over the choice of a sample
$S$ of size $m$, the following holds for any $h \in \sH$:
\begin{align*}
\sR_{\ell^{\rm{multi}}_{0-1}}(h) \leq \h \sR_{S}^{\brho}(h) + \frac{8 \Lambda_1 c}{m} \sqrt{\sum_{k = 1}^c \frac{m_k r_{k, \infty}^2}{\rho_k^2} \log(2d)} + \sqrt{\frac{\log \frac{1}{\delta}}{2m}}.
\end{align*}
\end{corollary}
\begin{corollary}
\label{cor:margin-bound-multi-kernel-2}
Consider $\sH_2 = \curl[\big]{(x, y) \mapsto w \cdot \Phi(x, y) \colon
  w \in \Rset^d, \norm*{w}_2 \leq \Lambda_2}$. Let $r_{k, 2} = \sup_{i
  \in I_k, y \in \sY} \norm*{\Phi(x_i, y)}_2$, for any $k \in
[c]$. Fix $\rho_{k} > 0$ for $k \in [c]$, then, for any $\delta > 0$,
with probability at least $1 - \delta$ over the choice of a sample $S$
of size $m$, the following holds for any $h \in \sH$:
\begin{equation*}
\sR_{\ell^{\rm{multi}}_{0-1}}(h) \leq \h \sR_{S}^{\brho}(h) + \frac{4 \sqrt{2} \Lambda_2 c}{m} \sqrt{\sum_{k = 1}^c \frac{m_k r_{k, 2}^2}{\rho_k^2}} + \sqrt{\frac{\log \frac{1}{\delta}}{2m}}.
\mspace{-8mu}
\end{equation*}
\end{corollary}
As with Theorem~\ref{thm:margin-bound-multi}, the bounds of these
corollaries can be generalized to hold uniformly for all $\rho_k \in
(0, 1]$ with $k \in [c]$, at the cost of additional terms
  $\sqrt{\frac{\log \log_2 \frac{2}{\rho_{k}}}{m}}$ for $k \in [c]$ by
  combining Theorem~\ref{thm:rad-kernel-1} or
  Theorem~\ref{thm:rad-kernel-2} with
  Theorem~\ref{thm:margin-bound-multi-uniform}, respectively.  Next,
  we describe an algorithm that can be derived directly from the
  theoretical guarantees presented above.

The guarantee of Corollary~\ref{cor:margin-bound-multi-kernel-2} and
it generalization to a uniform bound can be expressed as: for any
$\delta > 0$, with probability at least $1 - \delta$, for all $h \in
\sH_2 = \curl*{(x, y) \mapsto w \cdot \Phi(x, y) \colon w \in \Rset^d,
  \norm*{w}_2 \leq \Lambda_2}$,
\begin{align*}
\sR_{\ell^{\rm{multi}}_{0-1}}(h) \leq \frac{1}{m} \bracket*{ \sum_{k = 1}^c \sum_{i \in I_k} \max \paren*{0, 1 - \frac{\rho_{w}(x_i, k)}{\rho_k}}} + \frac{4 \sqrt{2} \Lambda_2 c}{m} \sqrt{\sum_{k = 1}^c \frac{m_k r_{k, 2}^2}{\rho_k^2}} + O\paren*{
\frac{1}{\sqrt{m}}}.    
\end{align*}
where $\rho_{w}(x, k) = w \cdot \Phi(x_i, k) - \max_{y' \neq k} \paren*{ w \cdot \Phi(x_i, y') }$, and we used the fact that the $\rho$-margin loss function is upper bounded by the $\rho$-hinge loss.

This suggests a regularization-based algorithm of the following form:
\begin{equation}
\min_{w \in \Rset^d} \lambda \norm*{w}^2 + \frac{1}{m} \bracket*{ \sum_{k = 1}^c \sum_{i \in I_k} \max \paren*{0, 1 - \frac{\rho_{w}(x_i, k)}{\rho_k}}},
\end{equation}
where, as in the binary classification, $\rho_k$s are chosen via
cross-validation. While $\rho_k$s can be chosen freely, the analysis
of lemma~\ref{lemma:D3} suggests concentrating the search around $\sfr
= \bracket[\bigg]{ \frac{m_k^{\frac{1}{3}} r_{k, 2}^{\frac{2}{3}}}
  {\ov r^{\frac{2}{3}}}}_{k \in [c]}$.

The above can be generalized to other multi-class surrogate loss
functions. In particular, when using the cross-entropy loss function
applied to the outputs of a neural network, the (multinomial) logistic
loss, our algorithm has the following form:
\begin{equation}
\min_{w \in \Rset^d} \lambda \norm*{w}^2 + \frac{1}{m} \sum_{k = 1}^c \sum_{i \in I_k} \log \bracket*{1 + \sum_{k' \neq k} e^{\frac{h(x_i, k') - h(x_i, k)}{\rho_{k}}}}.
\end{equation}
where $\rho_k$s are chosen via cross-validation.  When the number of
classes $c$ is large, we can restrict our search by considering the
same $\rho_k$ for classes with small representation, and distinct
$\rho_k$s for the top classes. Similar algorithms can be devised for
other $\norm*{\cdot}_p$ upper bounds on $w$, with $p \in [1,
  \infty)$. We can also derive a group-norm based generalization
  guarantee and corresponding algorithm.

  \subsection{Proof of Theorem~\ref{thm:rad-kernel-1} and
    Theorem~\ref{thm:rad-kernel-2}}
\label{app:rad-kernel}

\RadKernelOne*
\begin{proof}
The proof follows through a series of inequalities:
\begin{align*}
& \h \Rad_{S}^{\brho}(\sH_1)\\
& = \frac{1}{m} \E_{\e}\bracket*{\sup_{\norm*{w}_1 \leq \Lambda_1}  w \cdot \paren*{ \sum_{k = 1}^c \sum_{i \in I_k} \sum_{y \in \sY} \e_{iy} \frac{\Phi(x_i, y)}{\rho_{k}} }}\\
& \leq \frac{\Lambda_1}m \E_{\e}\bracket*{\norm*{ \sum_{k = 1}^c \sum_{i \in I_k} \sum_{y \in \sY} \e_{iy} \frac{\Phi(x_i, y)}{\rho_{k}}}_{\infty}} = \frac{\Lambda_1}m \E_{\e}\bracket*{\max_{j \in [d], s \in \curl*{-1, +1}}  s \sum_{k = 1}^c \sum_{i \in I_k} \sum_{y \in \sY} \e_{iy} \frac{\Phi_j(x_i, y)}{\rho_{k}}}\\ 
& \leq \frac{\Lambda_1}m \bracket*{2c \paren*{\sum_{k = 1}^c \frac{m_k r_{k, \infty}^2}{\rho_k^2}} \log(2d)}^{\frac12} = \frac{\Lambda_1 \sqrt{2c}}{m} \sqrt{\sum_{k = 1}^c \frac{m_k r_{k, \infty}^2}{\rho_k^2} \log(2d)}.
\end{align*}
The first inequality makes use of H\"older's inequality and the bound on
$\norm*{w}_1$, and the second one follows from the maximal inequality
and the fact that a Rademacher variable is 1-sub-Gaussian, and
$\sup_{i \in I_k, y \in \sY} \norm*{\Phi(x_i, y)}_{\infty} = r_{k,
  \infty}$.
\end{proof}

\RadKernelTwo*
\begin{proof}
The proof follows through a series of inequalities:
\begin{align*}
& \h \Rad_{S}^{\brho}(\sH_2)\\
& = \frac{1}{m} \E_{\e}\bracket*{\sup_{\norm*{w}_2 \leq \Lambda_2}  w \cdot \paren*{ \sum_{k = 1}^c \sum_{i \in I_k} \sum_{y \in \sY} \e_{iy} \frac{\Phi(x_i, y)}{\rho_{k}} }}\\
& \leq \frac{\Lambda_2}m \E_{\e}\bracket*{\norm*{ \sum_{k = 1}^c \sum_{i \in I_k} \sum_{y \in \sY} \e_{iy} \frac{\Phi(x_i, y)}{\rho_{k}}}_2} \leq \frac{\Lambda_2}m \bracket*{\E_{\e}\bracket*{\norm*{ \sum_{k = 1}^c \sum_{i \in I_k} \sum_{y \in \sY} \e_{iy} \frac{\Phi(x_i, y)}{\rho_{k}}}_2^2}}^{\frac12}\\ 
& \leq \frac{\Lambda_2}m \bracket*{\sum_{k = 1}^c \frac{1}{ \rho_{k}^2 } \sum_{i \in I_k} \sum_{y \in \sY} \norm*{\Phi(x_i, y)}^2_2}^{\frac12} \leq \frac{\Lambda_2}m \sqrt{c \sum_{k = 1}^c \frac{m_k r_{k, 2}^2}{\rho_k^2}}  = \frac{\Lambda_2 \sqrt{c}}{m} \sqrt{\sum_{k = 1}^c \frac{m_k r_{k, 2}^2}{\rho_k^2}}.
\end{align*}
The first inequality makes use of the Cauchy-Schwarz inequality and
the bound on $\norm*{w}_2$, the second follows by Jensen's inequality,
the third by $\E[\e_{iy} \e_{jy'}] = \E[\e_{iy}] \E[\e_{jy'}] = 0$ for
$i \neq j$ and $y \neq y'$, and the fourth one by $\sup_{i \in I_k, y
  \in \sY} \norm*{\Phi(x_i, y)}_2 = r_{k, 2}$.
\end{proof}

\end{document}